\documentclass[12pt,twoside]{article}
\newdimen\paravsp  \paravsp=1.3ex 
\usepackage{amsmath,amssymb}
\usepackage{wrapfig}  
\def\,{\mskip 3mu} \def\>{\mskip 4mu plus 2mu minus 4mu} \def\;{\mskip 5mu plus 5mu} \def\!{\mskip-3mu}
\def\dispmuskip{\thinmuskip= 3mu plus 0mu minus 2mu \medmuskip=  4mu plus 2mu minus 2mu \thickmuskip=5mu plus 5mu minus 2mu}
\def\textmuskip{\thinmuskip= 0mu                    \medmuskip=  1mu plus 1mu minus 1mu \thickmuskip=2mu plus 3mu minus 1mu}
\textmuskip
\def\beq{\dispmuskip\begin{equation}}    \def\eeq{\end{equation}\textmuskip}
\def\beqn{\dispmuskip\begin{displaymath}}\def\eeqn{\end{displaymath}\textmuskip}
\def\bqa{\dispmuskip\begin{eqnarray}}    \def\eqa{\end{eqnarray}\textmuskip}
\def\bqan{\dispmuskip\begin{eqnarray*}}  \def\eqan{\end{eqnarray*}\textmuskip}

\newtheorem{theorem}{Theorem}

\newtheorem{lemma}[theorem]{Lemma}

\newtheorem{open}[theorem]{Open~Problem}
\newenvironment{keywords}{\centerline{\bf\small
Keywords}\begin{quote}\small}{\par\end{quote}\vskip 1ex}
\newenvironment{proof}{{\vspace{\paravsp plus 0.5\paravsp minus 0.5\paravsp}\noindent\bf Proof.}}{}

\def\paradot#1{\vspace{\paravsp plus 0.5\paravsp minus 0.5\paravsp}\noindent{\bf\boldmath{#1.}}} 
\def\qmbox#1{{\quad\mbox{#1}\quad}} 
\def\blob{\noindent$\hspace*{1.8ex}\bullet\;\;$}  
\def\req#1{\eqref{#1}}          
\def\toinfty#1{\smash{\stackrel{#1\to\infty}{\longrightarrow}}} 
\def\eps{\varepsilon}           
\def\epstr{\epsilon}            
\def\nq{\hspace{-1em}}          
\def\qed{\hspace*{\fill}\rule{1.4ex}{1.4ex}$\quad$\\} 
\def\eoe{\hspace*{\fill} $\diamondsuit\quad$\\} 
\def\fr#1#2{{\textstyle{#1\over#2}}} 
\def\frs#1#2{{^{#1}\!/\!_{#2}}} 
\def\SetR{\mathbb{R}}           
\def\SetN{\mathbb{N}}           
\def\E{{\mathbb E}}             
\def\trp{{\!\top\!}}            
\def\d{\delta}
\def\g{\gamma}
\def\A{{\cal A}}
\def\O{{\cal O}}
\def\R{{\cal R}}
\def\S{{\cal S}}
\def\H{{\cal H}}
\def\Agent{\text{Agent}}
\def\Env{\text{Env}}
\def\MDP{\text{MDP}}
\def\t{\tilde}

\begin{document}

\title{\vspace{-4ex}
\vskip 2mm\bf\Large\hrule height5pt \vskip 4mm
Extreme State Aggregation Beyond MDPs
\vskip 4mm \hrule height2pt}
\author{{\bf Marcus Hutter}\\[3mm]
\normalsize Research School of Computer Science \\[-0.5ex]
\normalsize Australian National University \\[-0.5ex]
\normalsize Canberra, ACT, 0200, Australia \\
\normalsize \texttt{http://www.hutter1.net/}
}
\date{12 July 2014}
\maketitle

\begin{abstract}
We consider a Reinforcement Learning setup
where an agent interacts with an environment in
observation-reward-action cycles
without any (esp.\ MDP) assumptions on the environment. State
aggregation and more generally feature reinforcement learning is
concerned with mapping histories/raw-states to reduced/aggregated states.
The idea behind both is that the resulting reduced process (approximately)
forms a small stationary finite-state MDP, which can then be
efficiently solved or learnt. We considerably generalize existing
aggregation results by showing that even if the reduced process
is not an MDP, the (q-)value functions and (optimal) policies of
an associated MDP with same state-space size solve the original
problem, as long as the solution can approximately be represented
as a function of the reduced states. This implies an upper bound
on the required state space size that holds uniformly for all RL
problems. It may also explain why RL algorithms designed for MDPs
sometimes perform well beyond MDPs.
\def\contentsname{\centering\normalsize Contents}\setcounter{tocdepth}{1}
{\parskip=-2.7ex\tableofcontents}
\end{abstract}

\vspace*{-2ex}

\section{Introduction}\label{sec:Intro}

In {\em Reinforcement Learning} (RL) \cite{Sutton:98}, an {\em agent}
$\Pi$ takes actions in some {\em environment} $P$ and observes its
consequences and is rewarded for them. A well-understood and
efficiently solvable \cite{Puterman:94} and efficiently learnable
\cite{Strehl:09,Hutter:12pacmdp} case is where the environment is
(modelled as) a finite-state stationary {\em Markov Decision
Process} (MDP). Unfortunately most interesting real-world problems
$P$ are neither finite-state, nor stationary, nor Markov. One way
of dealing with this mismatch is to somehow transform the
real-world problem into a small MDP: {\em Feature Reinforcement
Learning} (FRL) \cite{Hutter:09phimdpx} and U-tree
\cite{McCallum:96} deal with the case of arbitrary unknown
environments, while state aggregation assumes the environment is a
large known stationary MDP \cite{Givan:03,Ferns:04}. The former maps
histories into states (Section~\ref{sec:PhiMDP}), the latter
groups raw states into aggregated states.

Here we follow the FRL approach and terminology, since it is
arguably most general: It subsumes the cases where the original
process $P$ is an MDP, a $k$-order MDP, a POMDP, and others (Section~\ref{sec:ExMDP}). 
Thinking in terms of histories also naturally stifles any temptation of a
naive frequency estimate of $P$ (no history ever repeats).
Finally we find the history vs state terminologically somewhat
neater than raw state vs aggregated state.

%
More importantly, we consider maps $\phi$ from histories to states
for which the reduced process $P_\phi$ is not (even approximately)
an MDP (Section~\ref{sec:Appr}). At first this seems to defeat the
original purpose, namely of reducing $P$ to a well-understood and
efficiently solvable problem class, namely small MDPs.
%
The main novel contribution of this paper is to
show that there is still an associated finite-state stationary MDP
$p$ whose solution (approximately) solves the original problem $P$,
as long as the solution can still be represented
(Section~\ref{sec:AAResults}).
%
Indeed, we provide an upper bound on the required state space size
that holds uniformly for all $P$ (Section~\ref{sec:ExSAgg}).
While these are interesting theoretical insights, it is a-priori
not clear whether they could by utilized to design (better) RL
algorithms.
%
We also show how to learn $p$ from experience
(Section~\ref{sec:RL}),
%
and sketch an overall learning algorithm and regret/PAC analysis
based on our main theorems (Section~\ref{sec:FRL}).
%
We briefly discuss how to relax one of the conditions in our main theorems
by permuting actions (Section~\ref{sec:Misc}).
%
We conclude with an outlook on future work and open problems
(Section~\ref{sec:Disc}).
%
A list of notation can be found in Appendix~\ref{app:Notation}.

The diagram below depicts the dependencies between our results:
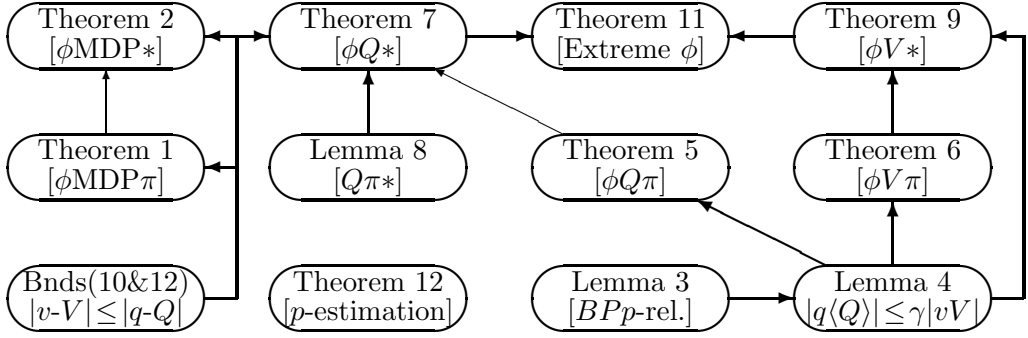
\begin{center}
\unitlength=2.4ex
\linethickness{0.4pt}
\begin{picture}(31,10)
\thicklines\small
\put(3,1){\oval(6,2)[cc]\makebox(0,0)[cb]{\raisebox{3pt}{Bnds(\ref{eqphipidelta}\&\ref{eqphistardelta})}}\makebox(0,0)[ct]{$|v$-$V|\leq|q$-$Q|$}}
\put(11,1){\oval(6,2)[cc]\makebox(0,0)[cb]{\raisebox{3pt}{Theorem~\ref{thm:pest}}}\makebox(0,0)[ct]{[$p$-estimation]}}
\put(19,1){\oval(6,2)[cc]\makebox(0,0)[cb]{\raisebox{3pt}{Lemma~\ref{lem:aBPp}}}\makebox(0,0)[ct]{[$BPp$-rel.]}}
\put(27,1){\oval(6,2)[cc]\makebox(0,0)[cb]{\raisebox{3pt}{Lemma~\ref{lem:aVvdQqd}}}\makebox(0,0)[ct]{$|q\langle Q\rangle|\leq\g|vV|$}}
\put(3,5){\oval(6,2)[cc]\makebox(0,0)[cb]{\raisebox{3pt}{Theorem~\ref{thm:phiMDPpi}}}\makebox(0,0)[ct]{[$\phi\MDP\pi$]}}
\put(11,5){\oval(6,2)[cc]\makebox(0,0)[cb]{\raisebox{3pt}{Lemma~\ref{lem:aQpistar}}}\makebox(0,0)[ct]{[$Q\pi*$]}}
\put(19,5){\oval(6,2)[cc]\makebox(0,0)[cb]{\raisebox{3pt}{Theorem~\ref{thm:aphiQpi}}}\makebox(0,0)[ct]{[$\phi Q\pi$]}}
\put(27,5){\oval(6,2)[cc]\makebox(0,0)[cb]{\raisebox{3pt}{Theorem~\ref{thm:aphiVpi}}}\makebox(0,0)[ct]{[$\phi V\pi$]}}
\put(3,9){\oval(6,2)[cc]\makebox(0,0)[cb]{\raisebox{3pt}{Theorem~\ref{thm:phiMDPstar}}}\makebox(0,0)[ct]{[$\phi\MDP*$]}}
\put(11,9){\oval(6,2)[cc]\makebox(0,0)[cb]{\raisebox{3pt}{Theorem~\ref{thm:aphiQstar}}}\makebox(0,0)[ct]{[$\phi Q*$]}}
\put(19,9){\oval(6,2)[cc]\makebox(0,0)[cb]{\raisebox{3pt}{Theorem~\ref{thm:Exphi}}}\makebox(0,0)[ct]{[Extreme~$\phi$]}}
\put(27,9){\oval(6,2)[cc]\makebox(0,0)[cb]{\raisebox{3pt}{Theorem~\ref{thm:aphiVstar}}}\makebox(0,0)[ct]{[$\phi V*$]}}
\put(6,1){\line(1,0){1}}\put(7,1){\line(0,1){8}}\put(7,5){\vector(-1,0){1}}\put(7,9){\vector(-1,0){1}}\put(7,9){\vector(1,0){1}}
\put(11,6){\vector(0,1){2}}
\put(14,9){\vector(1,0){2}}
\put(22,1){\vector(1,0){2}}
\put(24,9){\vector(-1,0){2}}
\put(25,2){\vector(-2,1){4}}
\put(27,2){\vector(0,1){2}}
\put(27,6){\vector(0,1){2}}
\put(30,1){\line(1,0){1}}\put(31,1){\line(0,1){8}}\put(31,9){\vector(-1,0){1}}
\thinlines
\put(3,6){\vector(0,1){2}}
\put(17,6){\vector(-2,1){4}}
\end{picture}
\end{center}

\section{Feature Markov Decision Processes ($\mathbf\Phi$MDP)}\label{sec:PhiMDP}

This section formally describes the setup of \cite{Hutter:09phimdpx}. It
consists of the agent-environment framework and maps $\phi$ from
observation-reward-action histories to MDP states. This arrangement
is called ``Feature MDP'' or short $\Phi$MDP. We use upper-case
letters $P$, $Q$, $V$, and $\Pi$ for the Probability, (Q-)Value, and
Policy of the original (agent-environment interactive) Process, and
lower-case letters $p$, $q$, $v$, and $\pi$ for the probability,
(q-)value, and policy of the (reduced/aggregated) MDP.

\paradot{Agent-environment setup \cite{Hutter:09phimdpx}}
We start with the standard agent-environment setup \cite{Russell:10}
in which an agent $\Pi$ interacts with an environment $P$. The
agent can choose from actions $a\in\A$ and the environment provides
observations $o\in\O$ and real-valued rewards
$r\in\R\subseteq[0;1]$ to the agent.
This happens in cycles $t=1,2,3,...$: At time $t$, after observing
$o_t$ and receiving reward $r_t$, the agent takes action $a_t$ based on
history
\beqn
  h_t ~:=~ o_1 r_1 a_1...o_{t-1} r_{t-1} a_{t-1} o_t r_t
  ~\in~ \H_t:=(\O\times\R\times\A)^{t-1}\times\O\times\R
\eeqn
Then the next cycle $t+1$ starts. The agent's objective is to
maximize its long-term reward.
To avoid integrals and densities, we assume spaces $\O$ and $\R$
are finite. They may be huge, so this is not really restrictive.
Indeed, the $\Phi$MDP framework has been specifically developed for
huge observation spaces. Generalization to continuous $\O$ and $\R$
is routine \cite{Hutter:09phidbn}. Furthermore we assume that $\A$
is finite and smallish, which is restrictive. Potential extensions
to continuous $\A$ are discussed in Section~\ref{sec:Disc}.

The agent and environment may be viewed as a pair of interlocking
functions of the history $\H:=(\O\times\R\times\A)^*\times\O\times\R$:
\bqan
  & & \nq\Env.~P:\H\times\A\leadsto\O\times\R, \qquad   P(o_{t+1}r_{t+1}|h_t a_t),\hspace{25ex} \\
  & & \nq\!\!\!\!\Agent~\Pi:\H\leadsto\A, \qquad\qquad \Pi(a_t|h_t) \qmbox{or} a_t=\Pi(h_t),
\eqan
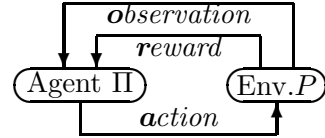
\begin{flushright}
\unitlength=1.2ex
\linethickness{0.4pt}
\begin{picture}(18,0)(0,-3) 
\thicklines\small
\put(3,3){\oval(8,2)[cc]\makebox(0,0)[cc]{Agent~$\Pi$}}
\put(15,3){\oval(6,2)[cc]\makebox(0,0)[cc]{Env.$P$}}
\put(3,2){\line(0,-1){2}}
\put(3,0){\line(1,0){12}}\put(9,0.5){\makebox(0,0)[cb]{$\boldsymbol a$\it\hspace{-1pt}ction}}
\put(15,0){\vector(0,1){2}}
\put(14,4){\line(0,1){2}}
\put(14,6){\line(-1,0){10}}\put(9,6){\makebox(0,0)[ct]{$\boldsymbol r$\it\hspace{-1.5pt}eward}}
\put(4,6){\vector(0,-1){2}}
\put(16,4){\line(0,1){4}}
\put(16,8){\line(-1,0){14}}\put(9,8){\makebox(0,0)[ct]{$\boldsymbol o$\it\hspace{-0.5pt}bservation}}
\put(2,8){\vector(0,-1){4}}
\end{picture}
\end{flushright}\vspace{-5ex}
where $\leadsto$ indicates that mappings $\to$ are in general stochastic.
We make no (stationarity or Markov or other) assumption on environment $P$.
For most parts, environment $P$ is assumed to be fixed,
so dependencies on $P$ will be suppressed.
For convenience and since optimal policies can be chosen to be deterministic,
we consider deterministic policies $a_t=\Pi(h_t)$ only.

\paradot{Value functions, optimal Policies, and history Bellman equations}
We measure the performance of a policy $\Pi$ in terms of the $P$-expected
$\g$-discounted reward sum ($0\leq\g<1$), called (Q-)Value of Policy
$\Pi$ at history $h_t$ (and action $a_t$)
\beqn
  V^\Pi(h_t)~:=~\E^\Pi[R_{t+1}|h_t] \qmbox{and}
  Q^\Pi(h_t,a_t)~:=~\E^\Pi[R_{t+1}|h_t a_t],
  \qmbox{where} R_t:=\smash{\sum_{\tau=t}^\infty} \g^{\tau-t}r_\tau
\eeqn
The optimal Policy and (Q-)Value functions are
\bqa
  \nonumber & & \nq V^*(h_t)~:=~\max_\Pi V^\Pi(h_t) \qmbox{and}
  Q^*(h_t,a_t)~:=~\max_\Pi Q^\Pi(h_t,a_t),
\\
  \label{VQPistar} & & \qmbox{where} \Pi^*~:\in\arg\max_\Pi V^\Pi(\epstr)
\eqa
The maximum over all policies $\Pi$ always exists
\cite{Hutter:14tcdiscx} but may not be unique, in which case
$\arg\max$ denotes the set of optimal policies and
$\Pi^*$ denotes a representative or the whole set
of optimal policies.
Despite being history-based we can write down (pseudo)recursive Bellman (optimality) equations
for the (optimal) (Q-)Values \cite[Sec.4.2]{Hutter:04uaibook}:
\bqa
  \label{eqQPi} Q^\Pi(h_t,a_t)\! &=& \!\!\sum_{\nq o_{t+1}r_{t+1}\nq}\!P(o_{t+1}r_{t+1}|h_t a_t)[r_{t+1}\!+\!\g V^\Pi(h_{t+1})],~
  V^\Pi(h_t) \!=\! Q^\Pi(h_t,\Pi(h_t))~~~~
\\
  \label{eqQstar} Q^*(h_t,a_t) &=& \!\sum_{\nq o_{t+1}r_{t+1}\nq}\! P(o_{t+1}r_{t+1}|h_t a_t)[r_{t+1}\!+\!\g V^*(h_{t+1})],~
  V^*(h_t) \!=\! \max_{a_t\in\A} Q^*(h_t,a_t)~~~~
\\ \label{eqPistar}
  \Pi^*(h_t) &\in& \arg\max_{a_t\in\A} Q^*(h_t,a_t)
\eqa
Unlike their classical state-space cousins (see below), they are {\em not} self-consistency equations:
The r.h.s.\ refers to a longer history $h_{t+1}$ which is always different
from the history $h_t$ on the l.h.s, which precludes any learning algorithm
based on estimating the frequency of state/history visits.
Still the recursions will be convenient for the mathematical development.

\paradot{From histories to states (\boldmath$\phi$)}
The space of histories is huge and unwieldy and no history ever
repeats. Standard ways of dealing with this are to define a similarity
metric on histories \cite{McCallum:96} or to aggregate histories
\cite{Hutter:09phimdpx}. We pursue the latter via a feature map
$\phi:\H\to\S$ which reduces histories $h_t\in\H$ to states
$s_t:=\phi(h_t)\in\S$. W.l.g.\ we assume that $\phi$ is surjective.
We also assume that state space $\S$ is finite;
indeed we are interested in small $\S$. This corresponds and
indeed is equivalent to a partitioning of histories
$\{\phi^{-1}(s):s\in\S\}$. Classical state aggregation usually
uses the partitioning view \cite{Givan:03,Ortner:07}, but the map notation is a bit
more convenient here.

The state $s_t$ is supposed to summarize all relevant information
in history $h_t$, which lower bounds the size of $\S$. We pass from
the complete history $o_1 r_1 a_1...o_n r_n$ to a `reduced'
history $s_1 r_1 a_1...s_n r_n$. Traditionally,
`relevant' means that the future is predictable from $s_t$ (and
$a_t$) alone, or technically that the reduced history forms a
Markov decision process. This is precisely the condition this paper
intends to lift (later).

\paradot{From histories to MDPs}
The probability of the successor states and rewards can be obtained by marginalization
\beq\label{eqPphi}
   P_\phi(s_{t+1}r_{t+1}|h_t a_t)
   ~:=~ \sum_{\t o_{t+1}:\phi(h_t a_t\t o_{t+1}r_{t+1})=s_{t+1}\nq\nq\nq\nq\nq} P(\t o_{t+1}r_{t+1}|h_t a_t)
\eeq
The reduced process $P_\phi$ is a Markov Decision Process, or Markov for short, if
$P_\phi$ only depends on $h_t$ through $s_t$, i.e.\ is the same for all histories
mapped to the same state. Formally
\beq\label{PphiMDP}
  P_\phi\in\MDP ~~~:\Longleftrightarrow~~~ \exists p : P_\phi(s_{t+1}r_{t+1}|\t h_t a_t)=p(s_{t+1}r_{t+1}|s_t a_t)~~\forall\phi(\t h_t)=s_t
\eeq
Here and elsewhere a quantifier such as $\forall\phi(\t h_t)=s_t$
shall mean: for all values of all involved variables consistent
with the constraint $\phi(\t h_t)=s_t$. The MDP $P_\phi$ is assumed to be
stationary, i.e.\ independent of $t$; another condition to be
lifted later.
Condition \req{PphiMDP} is essentially the stochastic bisimulation
condition generalized to histories and being somewhat
more restrictive regarding rewards \cite{Givan:03}: 
It is a condition on the reward distribution, while
\cite{Givan:03} constrains its expectation only. This could easily
be rectified but is besides the point of this paper. The
bisimulation metric \cite{Ferns:04} is an approximate version of
\req{PphiMDP}, which measures the deviation of $P_\phi$ from being
an MDP.

Many problems $P$ can be reduced
(approximately) to stationary MDPs \cite{Hutter:09phimdpx}:
Full-information {\em games} such as chess with static opponent are already Markov, %
classical {\em physics} is approximately 2nd-order Markov, %
(conditional) i.i.d.\ processes such as {\em Bandits} have counting sufficient statistics, %
and for a {\em POMDP planning} problem, the belief vector is Markov.

\paradot{Markov decision processes (MDP)}
We have used and continue to use upper-case letters $V$, $Q$, $\Pi$
for the general process $P$. We will use lower-case letters $v$,
$q$, $\pi$ for (stationary) MDPs $p$. We use $s$ and $a$ for the
current state and action, and $s'$ and $r'$ for successor state and
reward. Consider a stationary finite-state MDP
$p:\S\times\A\leadsto\S\times\R$ and stationary deterministic
policy $\pi:\S\to\A$.
Only in Section~\ref{sec:ExMDP} will this $p$ be given by \req{PphiMDP}, %
but {\em in general $p$ will be different from \req{PphiMDP}}.
In any case, the $p$-expected $\g$-discounted reward sum, called
(q-)value of (optimal) policy $\pi^{(*)}$ in MDP $p$,
are given by the Bellman (optimality) equations
\bqa
   \label{eqqpi}    q^\pi(s,a) &=& \sum_{s'r'}p(s'r'|sa)[r'\!+\!\g v^\pi(s')] \qmbox{and} v^\pi(s)=q^\pi(s,\pi(s)) \\
   \label{eqqstar}  q^*(s,a)   &=& \sum_{s'r'}p(s'r'|sa)[r'\!+\!\g v^*(s')] \qmbox{and} v^*(s)=\max_a q^*(s,a) \\
   \label{eqpistar} \pi^*(s) &\in& \arg\max_a q^*(s,a).~~~\text{Note:}~ v^\pi(s)\leq v^*(s),~q^\pi(s,a)\leq q^*(s,a)
\eqa
Using $p(s'r'|sa)=p(r'|sas')p(s'|sa)$ we could also rewrite
them in terms of transition matrix $p(s'|sa)$ and
expected reward $\E[r'|sa]$ \cite{Sutton:98}.

\paradot{More notation}
While our equations often assume or imply $s=s_t$, $a=a_t$,
$s'=s_{t+1}$, $r'=r_{t+1}$, (and $h_{t+1}=h_t ao'r'$) for some $t$,
technically $s,a,s',r'$ are {\em different} variables from all
variables in history $h_n= o_1 r_1 a_1...o_t r_t a_t
o_{t+1}r_{t+1}...a_n r_n$. Less prone to confusion are $o=o_t$,
$o'=o_{t+1}$, $h=h_t$, $h'=hao'r'$.

We call a function $f(h)$, piecewise constant or $\phi$-uniform
iff $f(h)=f(\t h)$ for all $\phi(h)=\phi(\t h)$.
Here and elsewhere $\forall\phi(h)=\phi(\t h)$ is short for
$\forall h,\t h:\phi(h)=\phi(\t h)$. Similarly $\forall s=\phi(h)$
is short for $\forall s,h:s=\phi(h)$. Etc.

The Iverson bracket, $[\![R]\!]:=1$ if $R$=true and $[\![R]\!]:=0$ if
$R$=false, denotes the indicator function.
Throughout, $\eps,\d\geq 0$ denote approximation accuracy.
Note that this includes the exact $=0$ case.

We now show that if $P$ reduces via $\phi$ to an MDP $p$, the
solution of these equations yields (Q-)Values and optimal Policy
of the original process $P$. This is not surprising and just a history-based
versions of classical state-aggregation results \cite{Givan:03}.
We state and prove them here, since notation and setup are somewhat different,
and proof ideas and fragments will be reused later.

\section{Exact Aggregation for $P_\phi\in\MDP$}\label{sec:ExMDP}

The following two theorems show that if $\phi$ reduces $P$ to a
stationary MDP via \req{eqPphi} and \req{PphiMDP}, then $V$ and $Q$
(and $\Pi^*$) essentially coincide with $v$ and $q$ (and $\pi^*$),
where policy $\Pi$ ($\Pi^*$) has to be assumed (will be shown)
constant within each partition $\phi^{-1}(s)$. This allows to
efficiently solve for (and learn in the case of unknown $P$) $V$
and $Q$ (and $\Pi^*$) in time polynomial in $\S$ by
solving/learning \req{eqqpi} (or \req{eqqstar} and \req{eqpistar})
instead of \req{eqQPi} (or \req{eqQstar} and \req{eqPistar}).

\begin{theorem}[\boldmath$\phi\MDP\pi$]\label{thm:phiMDPpi}
Let $\phi$ be a reduction such that $P_\phi\in\MDP$ reduces to MDP
$p$ defined in \req{PphiMDP}, and let $\Pi$ be some policy such
that $\Pi(h)=\Pi(\t h)$ for all $\phi(h)=\phi(\t h)$. Then
for all $a$ and $h$ it holds:
\beqn
  V^\Pi(h)=v^\pi(s) \qmbox{and} Q^\Pi(h,a)=q^\pi(s,a), \qmbox{where} \pi(s):=\Pi(h) \qmbox{and} s=\phi(h)
\eeqn
\end{theorem}
Note that $\pi(s)$ is well-defined, since $\phi$ is surjective and $\Pi(h)$ is the same for all $h\in\phi^{-1}(s)$.
The standard proof considers an $m$-horizon truncated MDP and induction on $m$ and $m\to\infty$.
Besides the adaptation to histories, the proof below is a slight variation that avoids such truncation and limit.
This style will be useful later. We explain all steps in detail here, since variations will be utilize later.

\begin{proof}
Let $\displaystyle \d:=\sup_{\nq s=\phi(h),a\nq }|q^\pi(s,a)-Q^\Pi(h,a)|$.
Using $a':=\pi(s')=\Pi(h')$ for $s'=\phi(h')$ and \req{eqQPi} and \req{eqqpi}
lets us bound the value difference
\beq\label{eqphipidelta}
  |v^\pi(s')\!-\!V^\Pi(h')| ~=~ |q^\pi(s',a')\!-\!Q^\Pi(h',a')| ~\leq~ \d
  ~~\forall s'=\phi(h')
\eeq
For any $a$ and $h$, this implies
\bqa
  \nonumber Q^\Pi(h,a) &\stackrel{(a)}=& \sum_{o'r'}P(o'r'|ha)[r'+\g V^\Pi(h')] ~~~~~~~~~~~~~~~~~~~ [h'=hao'r'] \\
  \label{prf:phiMDPpi} &\stackrel{(b)}\lessgtr& \sum_{s'r'}\sum_{o':\phi(h')=s'}P(o'r'|ha)[r'\!+\!\g(v^\pi(s')\pm\d)] \\
  \nonumber &\stackrel{(c)}=& \sum_{s'r'}P_\phi(s'r'|ha)[r'\!+\!\g v^\pi(s')]\pm\g\d \\
  \nonumber &\stackrel{(d)}=& \sum_{s'r'}p(s'r'|sa)[r'\!+\!\g v^\pi(s')]\pm\g\d ~~~~~~~~~~~~~~~~~ [s:=\phi(h)] \\
  \nonumber &\stackrel{(e)}=& q^\pi(s,a)\pm\g\d
\eqa
(a) is just \req{eqQPi}. %
In (b) we sum over all $o'$ by first summing over all $o'$ such
that $\phi(hao'r')=s'$ and then summing over all $s'$. We have also
upper/lower bounded $V^\Pi(h')$ via \req{eqphipidelta}. %
(c) is the definition \req{eqPphi} of $P_\phi$ and pulls out $\g\d$
using that probability $P_\phi$ sums to 1. %
(d) is the definition \req{PphiMDP} of $p$.
(e) is simply \req{eqqpi}.
The chain (\ref{prf:phiMDPpi}a-e) holds for all $s=\phi(h)$ and $a$, hence
\beqn
  \d ~=~ \sup_{\nq s=\phi(h),a\nq }|q^\pi(s,a)\!-\!Q^\Pi(h,a)| ~\leq~ \g\d
  ~~~\Rightarrow~~~ \d\leq 0
\eeqn
Hence $v^\pi(s)=V^\Pi(h)$ and $q^\pi(s,a)=Q^\Pi(h,a)$ for all $s=\phi(h)$ and $a$.
\qed\end{proof}

\begin{theorem}[\boldmath$\phi\MDP*$]\label{thm:phiMDPstar}
Let $\phi$ be a reduction such that $P_\phi\in\MDP$ reduces to MDP $p$ defined in \req{PphiMDP},
Then for all $a$ and $h$ it holds:
\beqn
  \Pi^*(h)=\pi^*(s) \qmbox{and} V^*(h)=v^*(s) \qmbox{and} Q^*(h,a)=q^*(s,a), \qmbox{where} s=\phi(h)
\eeqn
\end{theorem}

The core of the proof follows the same steps (\ref{prf:phiMDPpi}a-e) as for the previous theorem,
but the rest is slightly different. Additionally we have to show that $\Pi^*$ is
piecewise constant (in Theorem~\ref{thm:phiMDPpi} we assumed $\Pi$ was).

\begin{proof}
Let $\displaystyle \d:=\sup_{\nq s=\phi(h),a\nq }|q^*(s,a)-Q^*(h,a)|$.
We can bound the value difference
\beq\label{eqphistardelta}
  |v^*(s)\!-\!V^*(h)| \stackrel{(a)}= |\max_a q^*(s,a)\!-\!\max_a Q^*(h,a)|
  \stackrel{(b)}\leq \max_a|q^*(s,a)-Q^*(h,a)| \stackrel{(c)}\leq \d
  \,\forall s=\phi(h)
\eeq
(a) follows from the definitions \req{eqQstar} and \req{eqqstar}.
(b) follows from the following general elementary frequently used bound
\beq\label{eq3max}
  |\max_x f(x)-\max_x g(x)| ~\leq~ \max_x|f(x)-g(x)|
\eeq
(c) follows from the definition of $\d$.

One now can show that $Q^*(h,a)\lessgtr q^*(s,a)\pm\g\d$ for $s=\phi(h)$ by following
exactly the same steps as (\ref{eqphipidelta}a-e) just with
$\Pi$ and $\pi$ replaced by $*$ and using \req{eqphistardelta}
instead of \req{eqphipidelta}, and using the Bellman optimality
equations \req{eqQstar} and \req{eqqstar} instead of the
Bellman equations \req{eqQPi} and \req{eqqpi}.
Also as before, this implies $\d\leq\g\d$, hence $\d\leq 0$, hence
$v^*(s)=V^*(h)$ and $q^*(s,a)=Q^*(h,a)$ for all $s=\phi(h)$ and $a$.
Finally, the latter implies
$\pi^*(s)=\arg\max_a q^*(s,a)=\arg\max_a Q^*(h,a)=\Pi^*(h)$.
\qed\end{proof}

Approximate aggregation results if $P_\phi$ is approximately MDP
can also be derived \cite{Ferns:04}. The core results in the next section
show that aggregation is possible far beyond $P_\phi$ being approximately MDP.

\section{Approximate Aggregation for General $P$}\label{sec:Appr}

This section prepares for the main technical contribution of the
paper in the next section. The key quantity to relate original
and reduced Bellman equations is a form of stochastic inverse of
$\phi$, whose choice and analysis will be deferred to
Section~\ref{sec:RL}.

\paradot{Dispersion probability $B$}
Let $B_\phi:\S\times\A\leadsto\H$ be a probability distribution on
finite histories for each state-action pair such that $B_\phi(h|sa)=0$ if $s\neq\phi(h)$.
$B\equiv B_\phi$ may be viewed as a
stochastic inverse of $\phi$ that assigns non-zero probability only
to $h\in\phi^{-1}(s)$. The formal constraints we pose on $B$ are
\beq\label{eqaBdef}
  B(h|sa)\geq 0 \qmbox{and}
  \sum_{h\in\H}B(h|sa) = \sum_{\nq h:\phi(h)=s\nq}B(h|sa) = 1
  ~~\forall s,a
\eeq
This implicitly requires $\phi$ to be surjective, i.e.\
$\phi(\H)=\S$, which can always be made true by defining
$\S:=\S_\phi:=\phi(\H)$. Note that the sum is taken over histories
of any/mixed length. In general, $B$ is a somewhat weird
distribution, since it assigns probabilities to past and future
observations given the current state and action. The interpretation
and choice of $B$ does not need to concern us, except later when we
want to learn $p$.

The MDP requirement \req{PphiMDP} will be replaced by the following definition:
\bqa\label{eqapdef}
  p(s'r'|sa) &:=& \sum_{h\in\H}P_\phi(s'r'|ha)B(h|sa) \\
  \nonumber &\equiv& \sum_{t=1}^\infty\sum_{h_t\in\H_t}P_\phi(s_{t+1}=s',r_{t+1}=r'|h_t,a_t=a)B(h_t|sa)
\eqa
That is, the finite-state stationary MDP $p$ is built from feature
map $\phi$, dispersion probability $B$, and environment $P$: The
$p$-probability of observing state-reward pair $(s',r')$ from
state-action pair $(s,a)$ is defined as the $B$-average over all
histories $h$ consistent with $(s,a)$ of the $P_\phi$-probability
of observing $(s',r')$ (obtained from $P$ by $\phi$-marginalizing)
given history $h$ and action $a$.
The r.h.s.\ of the first line is merely shorthand for the second line.
Note that $sas'r'$ are fixed and do not appear in $h$ which ranges over histories $\H$ of all lengths.
It is easy to see that $p$ is a probability distribution, and it is
Markov by definition. If $P_\phi\in\MDP$, then definition \req{eqapdef}
coincides with $p$ defined in \req{PphiMDP}. In general, the MDP $p$,
depending on arbitrary $B$, is {\em not} the state distribution
induced by $P$ (and $\Pi$), which in general is non-Markov.
Note that $p$ is a stationary MDP for any $B$ satisfying \req{eqaBdef} and {\em any} $\phi$ and $P$.
We need the following lemmas:

\paradot{Some lemmas}
The first lemma establishes the key relation between $P$ and $p$
via $B$ used later to relate original history Bellman (optimality) equations
(\ref{eqQPi}--\ref{eqPistar}) with reduced state Bellman (optimality)
equations (\ref{eqqpi}--\ref{eqpistar}).

\begin{lemma}[\boldmath$B$-$P$-$p$ relation]\label{lem:aBPp}
For any function $f:\S\times\R\to\SetR$ and $p$ defined in
\req{eqapdef} in terms of $P$ via \req{eqPphi}, and $s':=\phi(h')$
and $h':=hao'r'$ it holds
\beqn
  \sum_{h\in\H}B(h|sa)\sum_{o'r'}P(o'r'|ha)f(\mathop{s'}\limits_{\textstyle\uparrow\atop\makebox[0ex]{\footnotesize depends on $hao'r'$}},r')
  ~=~ \sum_{s'r'}p(s'r'|sa)f(s',r')
\eeqn
\end{lemma}

\begin{proof}\vspace{-3ex}
\bqan
  & & \sum_{\smash{h\in\H}}B(h|sa)\sum_{o'r'}P(o'r'|ha)f(s',r') \\
  &\stackrel{(a)}=& \sum_{h\in\H}B(h|sa)\sum_{s'r'}\sum_{o':\phi(h')=s'}P(o'r'|ha)f(s',r') \\
  &\stackrel{(b)}=& \sum_{h\in\H}B(h|sa)\sum_{s'r'}P_\phi(s'r'|ha)f(s',r') \\
  &\stackrel{(c)}=& \sum_{s'r'}p(s'r'|sa)f(s',r')
\eqan
In (a) we sum over all $o'$ by first summing over all $o'$ such
that $\phi(hao'r')=s'$ and then summing over all $s'$. %
In (b) we used the definition \req{eqPphi} of $P_\phi$. %
In (c) we used the definition \req{eqapdef} of $p$. %
\qed\end{proof}

Inequalities \req{eqphipidelta} and \req{eqphistardelta} trivially
bound $v-V$ differences in terms of $q-Q$ differences:
$|v-V|\leq\max_a|q-Q|$.
The following lemma shows that a reverse
holds in expectation, i.e.\ $|q-\langle Q\rangle_B|\leq\g|v-V|$.
The expectation can (only) be dropped if $Q$ is constant for
$h\in\phi^{-1}(s)$. Formally define
\beq\label{eqEfB}
  \langle f(h,a)\rangle_B ~:=~ \sum_{\t h\in\H}B(\t h|sa)f(\t h,a), \qmbox{where} s:=\phi(h)
\eeq
That is, $\langle f(h,a)\rangle_B$ takes a $B$-average over all
$\t h$ that $\phi$ maps to the same state as $h$. For
convenience we will drop the tilde, which we can do if we declare
$s:=\phi(h)$ to refer to the `global' $h$ in $\langle f(h,a)\rangle_B$ and not
to the `local' variable in the $h\in\H$ sum.

\begin{lemma}[\boldmath$|q-\langle Q\rangle|\leq\g|v-V|$]\label{lem:aVvdQqd} 
For any $P$, $\phi$, $B$, define $p$ via \req{eqapdef} and \req{eqPphi}. \\
(i) If $|v^\pi(s)-V^\Pi(h)|\leq\d$ $\forall s=\phi(h)$ \\
\hspace*{4ex}then $|q^\pi(s,a)-\langle Q^\Pi(h,a)\rangle_B|\leq\g\d$ $\forall s=\phi(h)~\forall a$. \\
(ii) If $|v^*(s)-V^*(h)|\leq\d$ $\forall s=\phi(h)$ \\
\hspace*{3.7ex}then $|q^*(s,a)-\langle Q^*(h,a)\rangle_B|\leq\g\d$ $\forall s=\phi(h)~\forall a$.
\end{lemma}

\begin{proof} (i) Let $s:=\phi(h)$ and $h':=hao'r'$ and $s':=\phi(h')$. Then
\bqan
  \langle Q^\Pi(h,a)\rangle_B &\stackrel{\req{eqEfB}}\equiv& \sum_{h\in\H}B(h|sa)Q^\Pi(h,a) \\
  &\stackrel{\req{eqQPi}}=& \sum_{h\in\H}B(h|sa)\sum_{o'r'}P(o'r'|ha)[r'+\g V^\Pi(h')] \\
  &\stackrel{(a)}\lessgtr& \sum_{h\in\H}B(h|sa)\sum_{o'r'}P(o'r'|ha)[r'\!+\!\g(v^\pi(s')\pm\d)] \\
  &\stackrel{Lem.\ref{lem:aBPp}}=& \sum_{s'r'}p(s'r'|sa)[r'\!+\!\g v^\pi(s')]\pm\g\d \\
  &\stackrel{\req{eqqpi}}=& q^\pi(s,a)\pm\g\d
\eqan
In (a) we used the assumption (i) of the Lemma. %
The derived upper and lower bounds imply $|q^\pi(s,a)-\langle
Q^\Pi(h,a)\rangle_B|\leq\g\d$ (for all $s=\phi(h)$ and $a$).

(ii) follows the same steps except with $\Pi$ and $\pi$ replaced by $\Pi^*$ and $\pi^*$,
and using \req{eqQstar} and \req{eqqstar} instead of \req{eqQPi} and \req{eqqpi} to justify the steps.
Note that in general $\Pi^*\neq \pi^*$!
\qed\end{proof}

\section{Approximate Aggregation Results}\label{sec:AAResults}

This section contains the main technical contribution of the paper.
We show that histories (or raw states) can be aggregated and
modeled by an MDP even if the true aggregated process is actually
not an MDP. A necessary condition for successful aggregation is of
course that the quantities of interest, namely (Q-)Value functions
and Policies can be represented as functions of the aggregated
states. The results in this section roughly show that this
necessary condition, which is significantly weaker than the MDP
requirement, is also sufficient. All but one result also holds for
approximate aggregation, i.e.\ approximate conditions lead to
approximate reductions. We also lift the stationarity assumption.

\begin{itemize}\parskip=0ex\parsep=0ex\itemsep=0ex
\item Theorem~\ref{thm:aphiQpi} shows how (approximately) $\phi$-uniform
$Q^\Pi$ and $\Pi$ can be obtained from the reduced Bellman
equations \req{eqqpi}.
\item Theorem~\ref{thm:aphiVpi} weakens the assumptions and conclusions
to (approximately) $\phi$-uniform $V^\Pi$ and $\Pi$.
\item Theorem~\ref{thm:aphiQstar} shows that for (approximately)
$\phi$-uniform $Q^*$, the optimal policy is (approximately)
$\phi$-uniform, and (an approximation of it) can be obtained via
the reduced Bellman optimality equations \req{eqqstar}.
\item Theorem~\ref{thm:aphiVstar} shows that for (approximately)
$\phi$-uniform $V^*$ and $\Pi^*$ we can obtain similar but somewhat
weaker results. The proof of the latter involves extra
complications not present in the other three proofs. Indeed,
whether the arguably most desirable bound holds is Open~Problem~\ref{open:aphiVstar}.
\end{itemize}
Note that all theorems crucially differ in their conditions and
conclusions.

\begin{theorem}[\boldmath$\phi Q\pi$]\label{thm:aphiQpi}
For any $P$, $\phi$, and $B$, define $p$ via \req{eqapdef} and \req{eqPphi}.
Let $\Pi$ be some policy such that
$\Pi(h)=\Pi(\t h)$ and $|Q^\Pi(h,a)-Q^\Pi(\t h,a)|\leq\eps$ for all
$\phi(h)=\phi(\t h)$ and all $a$.
Then for all $a$ and $h$ it holds:
\bqan
  |Q^\Pi(h,a)-q^\pi(s,a)|\leq{\eps\over 1-\g} ~~~&\mbox{and}&~~~ |V^\Pi(h)-v^\pi(s)|\leq{\eps\over 1-\g}, \\
  \qmbox{where} \pi(s):=\Pi(h) &\mbox{and}& s=\phi(h)
\eqan
\end{theorem}

\begin{proof}
Let $\displaystyle \d:=\sup_{\nq s=\phi(h),a\nq }|q^\pi(s,a)-Q^\Pi(h,a)|$.
Then $|v^\pi(s)-V^\Pi(h)|\leq\d$ $\forall s=\phi(h)$ by \req{eqphipidelta},
\beqn
  \qmbox{hence} |q^\pi(s,a)-\langle Q^\Pi(h,a)\rangle_B|\leq\g\d ~~ \forall s=\phi(h),a
\eeqn
by Lemma~\ref{lem:aVvdQqd}i. By assumption on $Q^\Pi$ and $B$, for $s=\phi(h)$ we have
\beqn
  \langle Q^\Pi(h,a)\rangle_B
  ~\equiv~ \sum_{\nq\t h\in\H:\phi(\t h)=s\nq\nq}B(\t h|sa)Q^\Pi(\t h,a)
  ~\lessgtr~ \sum_{\nq\t h\in\H:\phi(\t h)=s\nq\nq}B(\t h|sa)[Q^\Pi(h,a)\pm\eps]
  ~=~ Q^\Pi(h,a)\pm\eps
\eeqn
Together this implies $|q^\pi(s,a)-Q^\Pi(h,a)|\leq\g\d+\eps$, hence $\d\leq\g\d+\eps$, hence $\d\leq{\eps\over 1-\g}$.
\qed\end{proof}

\begin{theorem}[\boldmath$\phi V\pi$]\label{thm:aphiVpi}
For any $P$, $\phi$, and $B$, define $p$ via \req{eqapdef} and \req{eqPphi}.
Let $\Pi$ be some policy such that
$\Pi(h)=\Pi(\t h)$ and $|V^\Pi(h)-V^\Pi(\t h)|\leq\eps$ for all
$\phi(h)=\phi(\t h)$.
Then for all $a$ and $h$ it holds:
\bqan
  |V^\Pi(h)-v^\pi(s)|\leq{\eps\over 1-\g} ~~~&\mbox{and}&~~~ |q^\pi(s,a)-\langle Q^\Pi(h,a)\rangle_B|\leq{\eps\g\over 1-\g} \\
  \qmbox{where} \pi(s):=\Pi(h) &\mbox{and}& s=\phi(h)
\eqan
\end{theorem}

\begin{proof}
Let $\displaystyle \d:=\sup_{\nq s=\phi(h),a\nq }|v^\pi(s)-V^\Pi(h)|$,
fix some $s=\phi(h)$, and let $a^\pi:=\Pi(h)$. Now
\bqa\label{eqQVp}
  \nonumber \langle Q^\Pi(h,a^\pi)\rangle_B
  &\equiv& \sum_{\t h\in\H:\phi(\t h)=s\nq\nq\nq}B(\t h|sa^\pi) Q^\Pi(\t h,a^\pi)
  ~\stackrel{(a)}=~ \sum_{\t h\in\H:\phi(\t h)=s\nq\nq\nq}B(\t h|sa^\pi) V^\Pi(\t h) \\
  &\lessgtr& \sum_{\t h\in\H:\phi(\t h)=s\nq\nq\nq}B(\t h|sa^\pi)[V^\Pi(h)\pm\eps]
  ~=~ V^\Pi(h)\pm\eps
\eqa
where (a) follows from $a^\pi=\Pi(h)=\Pi(\t h)$ and $Q^\Pi(\t
h,\Pi(\t h))=V^\Pi(\t h)$. By Lemma~\ref{lem:aVvdQqd}i
we have
\beq\label{eqpVpp}
  |q^\pi(s,a)-\langle Q^\Pi(h,a)\rangle_B|\leq\g\d ~~ \forall s=\phi(h),a
\eeq
We also have $q^\pi(s,a^\pi)=q^\pi(s,\pi(s))=v^\pi(s)$ from
\req{eqqpi}. Together with \req{eqQVp} and \req{eqpVpp} for $a=a^\pi$ this yields
\beqn
  |v^\pi(s)-V^\Pi(h)| ~\leq~ |v^\pi(s)\!-\!\langle Q^\Pi(h,a^\pi)\rangle_B|
  + |\langle Q^\Pi(h,a^\pi)\rangle_B\!-\!V^\Pi(h)| ~\leq~ \g\d+\eps
\eeqn
hence $\d\leq\g\d+\eps$ by the definition of $\d$, hence $\d\leq{\eps\over 1-\g}$.
Note that while $|q^\pi(s,a^\pi)-Q^\Pi(h,a^\pi)|\leq{\eps\g\over 1-\g}$,
in general $|q^\pi(s,a)-Q^\Pi(h,a)|\not\leq{\eps\g\over 1-\g}$ for $a\neq a^\pi$.
\qed\end{proof}

\begin{wrapfigure}{r}{32ex}
\unitlength=1.5ex
\linethickness{0.4pt}
\begin{picture}(21,10.5)
\thicklines
\put(7,9){\circle{3}}\put(7,9){\makebox(0,0)[cc]{$00$}}
\put(13,9){\circle{3}}\put(13,9){\makebox(0,0)[cc]{$01$}}
\put(7,3){\circle{3}}\put(7,3){\makebox(0,0)[cc]{$10$}}
\put(13,3){\circle{3}}\put(13,3){\makebox(0,0)[cc]{$11$}}
\put(5,9){\makebox(0,0)[rc]{$r'={\g/2\over 1+\g}$}}
\put(15,9){\makebox(0,0)[lc]{$r'={1-\g/2\over 1+\g}$}}
\put(5,3){\makebox(0,0)[rc]{$r'=0$}}
\put(15,3){\makebox(0,0)[lc]{$r'=1$}}
\put(8.2,10){\vector(1,0){3.6}}\put(11.7,8.2){\vector(-1,0){3.4}}\put(10,9){\makebox(0,0)[cc]{\footnotesize$1\!/\!2$}}
\put(7,7.5){\vector(0,-1){3}}\put(6.5,6){\makebox(0,0)[rc]{\footnotesize$1\!/\!2$}}
\put(13,7.5){\vector(0,-1){3}}\put(13.5,6){\makebox(0,0)[lc]{\footnotesize$1\!/\!2$}}
\put(8.06,4.06){\vector(1,1){3.88}}\put(9,6){\makebox(0,0)[cc]{\footnotesize$1$}}
\put(11.94,4.06){\vector(-1,1){3.88}}\put(11,6){\makebox(0,0)[cc]{\footnotesize$1$}}
\put(7,1.5){\makebox(0,0)[ct]{$\underbrace{\rule{8ex}{0ex}}_{s=0}$}}
\put(13,1.5){\makebox(0,0)[ct]{$\underbrace{\rule{8ex}{0ex}}_{s=1}$}}
\end{picture}
\end{wrapfigure}
\paradot{Example}
Consider a process $P$ which itself is an MDP in the observations
with transition matrix $T$ and reward function $R$, i.e.\
$P(o'r'|ha)=T_{oo'}^a R_{oo'}^{ar'}$. The example on the right has
the special form $P(o'r'|ha)=T_{oo'}\cdot[\![r'=R(o)]\!]$. It is an
action-independent Markov process $T$ with deterministic reward
function $R$, which can be read off from the diagram. Observation
space is $\O=\{00,01,10,11\}$. Consider reduction
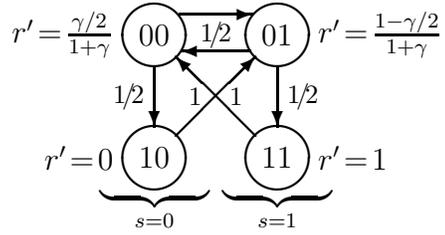
\beqn
  s_t ~:=~ \phi(h_t) ~:=~ \left\{ { 0~~\text{if}~~o_t=00~\text{or}~10 \atop
                                    1~~\text{if}~~o_t=01~\text{or}~11 } \right\}
  ~\in~ \S ~:=~ \{0,1\}
\eeqn
The reduced process $P_\phi$ is not (even approximately) Markov:
\bqan
  P_\phi(s'=0|o=00) &=& T_{00,00}+T_{00,10} ~=~ 0 ~+~ 1/2 ~=~ 1/2 ~~~ \smash{\raisebox{-2ex}{\Large$\not=$}}\\
  P_\phi(s'=0|o=10) &=& T_{10,00}+T_{10,10} ~=~ 0 ~+~~~0 ~~=~~ 0 ~~~
\eqan
That is, $P$ violates the bisimulation condition \cite{Givan:03}, and
raw states $00$ and $10$ have a large bisimulation distance
\cite{Ferns:04,Ortner:07}.
On the other hand, the (Q-)Value function
$V(o_t):=V^\pi(h_t)=Q^\pi(h_t,a_t)\forall a_t$ can easily be verified to be
\beqn
  V(00) ~=~ V(10) ~=~ {\g\over 1-\g^2} ~~\qmbox{and}~~ V(01) ~=~ V(11) ~=~ {1\over 1-\g^2}
\eeqn
That is, $V$ and $Q$ are $\phi$-uniform. The conditions of
Theorems~\ref{thm:aphiQpi}~and~\ref{thm:aphiVpi} are satisfied
exactly ($\eps=0$), and hence the four raw states $\O$ can be
aggregated into two states $\S$ despite $P_\phi\not\in\MDP$
(the policy is irrelevant and can be chosen constant).\eoe

We now turn from the fixed policy case to similar theorems for optimal policies.

\begin{theorem}[\boldmath$\phi Q*$]\label{thm:aphiQstar}
For any $P$, $\phi$, and $B$, define $p$ via \req{eqapdef} and \req{eqPphi}.
Assume $|Q^*(h,a)-Q^*(\t h,a)|\leq\eps$ for all $\phi(h)=\phi(\t h)$ and all $a$.
Then for all $a$ and $h$ and $s=\phi(h)$ it holds:
\bqan
  &(i)& |Q^*(h,a)-q^*(s,a)|\leq{\eps\over 1-\g} \qmbox{and} |V^*(h)-v^*(s)|\leq{\eps\over 1-\g},
\\
  &(ii)& 0 ~\leq~ V^*(h)-V^{\t\Pi}(h) ~\leq~ {2\eps\over(1-\g)^2}, \qmbox{where} \t\Pi(h):=\pi^*(s)
\\
  &(iii)& \text{If $\eps=0$ then $\Pi^*(h)=\pi^*(s)$}
\eqan
\end{theorem}

\begin{proof}
(i) The proof follows the same steps as the proof of Theorem~\ref{thm:aphiQpi}, replacing
all $\Pi$ and $\pi$ by $*$ and using \req{eqphistardelta} instead of \req{eqphipidelta}
and Lemma~\ref{lem:aVvdQqd}ii instead of Lemma~\ref{lem:aVvdQqd}i to justify the steps.

(iii) If $\eps=0$, then $Q^*(h,a)=q^*(s,a)$ by (i) implies $\Pi^*(h)=\pi^*(s)$, where it is worthwhile
to carefully check that the latter has actually not been used inadvertently in proving the former.
Cf.\ the next theorem and proof.

(ii) For $s=\phi(h)$ and $\t a:=\t\Pi(h)=\pi^*(s)$,
\beqn
  V^*(h)-{\eps\over 1-\g} ~\stackrel{(i)}\leq~ v^*(s) ~\stackrel{\req{eqqstar}}=~ q^*(s,\t a) ~\stackrel{(i)}\leq~ Q^*(h,\t a)+{\eps\over 1-\g}
\eeqn
which implies $Q^*(h,\t\Pi(h))\geq V^*(h)-{2\eps\over 1-\g}$.
The claim now follows from the next Lemma~\ref{lem:aQpistar} below.
\qed\end{proof}

The following lemma shows that if replacing the first action after $h$ of the
optimal policy $\Pi^*$ by the action provided by $\Pi$ thereafter
following $\Pi^*$ is at most $\eps$-suboptimal, then always using
$\Pi$ is at most ${\eps\over 1-\g}$-suboptimal.

\begin{lemma}[\boldmath$Q\pi*$]\label{lem:aQpistar}
If $Q^*(h,\Pi(h))\geq V^*(h)-\eps$ for all $h$ for some policy $\Pi$, then for all $h$ and $a$
\beqn
  0 ~\leq~ Q^*(h,a)-Q^\Pi(h,a) ~\leq~ {\eps\g\over 1-\g} ~~\qmbox{and}~~
  0 ~\leq~ V^*(h)-V^\Pi(h) ~\leq~ {\eps\over 1-\g}
\eeqn
\end{lemma}

\begin{proof}
Let $\displaystyle\d:=\sup_{h,a}[Q^*(h,a)-Q^\Pi(h,a)]$. This implies
\beq\label{eqaVstarV}
  0 ~\stackrel{(a)}\leq~ V^*(h)-V^\Pi(h)
    ~\stackrel{(b)}\leq~ \eps + Q^*(h,\Pi(h)) - Q^\Pi(h,\Pi(h))
    ~\stackrel{(c)}\leq~ \eps+\d
\eeq
(a) follows from \req{VQPistar}; (b) by assumption; and (c) by definition of $\d$ for $a=\Pi(h)$.
Now for any $a$ and $h$, this implies
\bqan
  Q^\Pi(h,a) &\stackrel{\req{VQPistar}}\leq& Q^*(h,a) ~\stackrel{\req{eqQstar}}=~ \sum_{o'r'}P(o'r'|ha)[r'+\g V^*(h')] ~~~~~~~~~~~~~~~~~ [h'=hao'r'] \\
  &\stackrel{\req{eqaVstarV}}\leq& \sum_{o'r'}P(o'r'|ha)[r'+\g(V^\Pi(h')+\eps+\d)]
  ~\stackrel{\req{eqQPi}}=~ Q^\Pi(h,a)+\g(\eps+\d)
\eqan
Hence $\d\leq\g(\eps+\d)$, hence $\d\leq{\eps\g\over 1-\g}$.
\qed\end{proof}

\begin{theorem}[\boldmath$\phi V*$]\label{thm:aphiVstar}
For any $P$, $\phi$, and $B$, define $p$ via \req{eqapdef} and \req{eqPphi}.
Assume $\Pi^*(h)=\Pi^*(\t h)$ and $|V^*(h)-V^*(\t h)|\leq\eps$ for all
$\phi(h)=\phi(\t h)$.
Then for all $a$ and $h$ and $s=\phi(h)$ it holds:
\bqan
  &(i)& |V^*(h)-v^*(s)|\leq{3\eps\over(1-\g)^2} \qmbox{and} |q^*(s,a)-\langle Q^*(h,a)\rangle_B|\leq{3\eps\g\over(1-\g)^2},
\\
  &(ii)& \text{If $\eps=0$ then $\Pi^*(h)=\pi^*(s)$}
\eqan
\end{theorem}

The proof actually implies the stronger lower bound
$V^*(h)-v^*(s)\geq {3\eps\over 1-\g}$ and similarly for $Q^*$, but we
do not know whether the upper bound can be improved.

\begin{proof}
While proofs start to get routine, here is a warning that care is
in order when recycling similar proofs. Theorem~\ref{thm:aphiVpi}
relies on the assumption that $\pi(s)=\Pi(h)$ for $s=\phi(h)$,
while we were lucky that the proof of
Theorem~\ref{thm:aphiQstar} worked without knowing $\pi^*(s)=\Pi^*(h)$
in advance. Here we have to work a bit harder.

Let us define $a^0:=\pi^0(s):=\Pi^*(h)$ for $s=\phi(h)$.
The Bellman equation for policy $\pi^0$ is
\beq\label{eqaq0}
  q^{\pi^0}(s,a) ~=~ \sum_{s'r'}p(s'r'|sa)[r'\!+\!\g v^{\pi^0}(s')] \qmbox{and} v^{\pi^0}(s)=q^{\pi^0}(s,{\pi^0}(s))
\eeq
At this stage $\pi^0$ may well be different from $\pi^*$, since
$\pi^*$ satisfies \req{eqqstar}, not \req{eqaq0}, but we will now show that it actually does.
First note that
\beq\label{eqaq0a0}
  q^{\pi^0}(s,a^0) ~=~ v^{\pi^0}(s) ~\lessgtr~ V^{\Pi^*}(h)\pm{\eps\over 1-\g} ~=~ V^*(h)\pm{\eps\over 1-\g}
\eeq
where the bounds follow from Theorem~\ref{thm:aphiVpi} applied to $\Pi:=\Pi^*$ (with $\pi=\pi^0$).
For general $a$ we only get an upper bound:
\bqa\label{eqaq0a}
  q^{\pi^0}(s,a)-{\eps\g\over 1-\g} ~\stackrel{Thm.\ref{thm:aphiVpi}}\leq~ \langle Q^{\Pi^*}(h,a)\rangle_B
            &\stackrel{\req{eqEfB}}=& \sum_{h\in\H}B(h|sa)Q^*(h,a) \\
  \nonumber ~\stackrel{\req{eqPistar}}\leq~ \sum_{h\in\H}B(h|sa)Q^*(h,\Pi^*(h))
  \nonumber &\mathop{=}\limits_{\req{eqQstar}}^{\req{eqaBdef}}& \sum_{\t h\in\H:\phi(\t h)=s\nq\nq\nq}B(\t h|sa)V^*(\t h) \\
  \nonumber ~\stackrel{(a)}\leq~ \sum_{\nq\t h\in\H:\phi(\t h)=s\nq\nq}B(\t h|sa)[V^*(h)+\eps]
  \nonumber &\stackrel{\req{eqaBdef}}=& V^*(h)+\eps
\eqa
(a) uses the theorem's assumption on $V^*(h)$.
Together, \req{eqaq0a0} and \req{eqaq0a} imply
\beq\label{eqavpi0Vstar}
  v^{\pi^0}(s) ~\stackrel{\req{eqaq0}}=~ q^{\pi^0}(s,a^0)
  ~\leq~ \max_a q^{\pi^0}(s,a)
  ~\stackrel{\req{eqaq0a}}\leq~ V^*(h)+{\eps\over 1-\g}
  ~\stackrel{\req{eqaq0a0}}\leq~ v^{\pi^0}(s)+{2\eps\over 1-\g}
\eeq

{\bf(ii)} For $\eps=0$, the previous equation implies $v^{\pi^0}(s)=\max_a q^{\pi^0}(s,a)$, hence
\req{eqaq0} can be rewritten as
\beqn
  q^{\pi^0}(s,a) ~=~ \sum_{s'r'}p(s'r'|sa)[r'\!+\!\g v^{\pi^0}(s')] \qmbox{and} v^{\pi^0}(s)=\max_a q^{\pi^0}(s,a)
\eeqn
This shows that $(q^{\pi^0},v^{\pi^0})$ satisfies the same Bellman {\em optimality} equation
as $(q^*,v^*)$ does. Since it has a unique solution, we must have
$q^{\pi^0}\equiv q^*$ and $v^{\pi^0}\equiv v^*$ and $\pi^*\equiv\pi^0$,
which for $s=\phi(h)$ implies $\Pi^*(h)=\pi^*(s)$ by definition of $\pi^0$.
It also implies
$V^*(h)=v^*(s)$ by \req{eqaq0a0}, and $q^*(s,a)=\langle Q^*(h,a)\rangle_B$ by
Lemma~\ref{lem:aVvdQqd}ii, i.e.\ the $\eps=0$ version of (i).

{\bf(i)} We now continue with the general $\eps>0$ case. For all $s$ and $a$ we have
\beqn
  0 ~\stackrel{\req{eqpistar}}\leq~ q^*(s,a)-q^{\pi^0}(s,a)
  ~\mathop{=}^{\req{eqqpi}}_{\req{eqqstar}}~ \smash{\sum_{s'r'}}p(s'r'|sa)\g(v^*(s')-v^{\pi^0}(s'))
  ~\stackrel{(a)}\leq~ \g\max_{s'}\{v^*(s')-v^{\pi^0}(s')\}
\eeqn
\beqn
  0 \stackrel{\req{eqpistar}}\leq v^*(s)\!-\!v^{\pi^0}(s)
  \mathop{\leq}^{\req{eqavpi0Vstar}}_{\req{eqqstar}} \max_a q^*(s,a)\!-\!\max_a q^{\pi^0}(s,a)\!+\!{\textstyle{2\eps\over 1-\g}}
  \stackrel{\req{eq3max}}\leq \max_a\{q^*(s,a)\!-\!q^{\pi^0}(s,a)\}\!+\!{\textstyle{2\eps\over 1-\g}}
\eeqn
In (a) we have upper bounded the $p$-expectation by the maximum.
Together this gives
\bqa
  \nonumber\max_s\{v^*(s)-v^{\pi^0}(s)\} &\leq& \g\max_s\{v^*(s)-v^{\pi^0}(s)\}+{2\eps\over 1-\g} \\
  \nonumber\Rightarrow~~~ \max_s\{v^*(s)-v^{\pi^0}(s)\} &\leq& {2\eps\over(1-\g)^2}
\\
  \label{eqapVsph}\text{Hence for } s\!=\!\phi(h):~ V^*(h)\!-\!{\eps\over 1-\g}
  \!\! &\stackrel{\req{eqaq0a0}}\leq& \!\! v^{\pi^0}(s)
  \stackrel{\req{eqpistar}}\leq v^*(s) \stackrel\nwarrow\leq v^{\pi^0}(s)\!+\!{2\eps\over(1-\g)^2}~~~~~~~
\\
  \nonumber ~\stackrel{\req{eqaq0a0}}\leq~ V^*(h)+{\eps\over 1-\g}+{2\eps\over(1-\g)^2}
  &\leq& V^*(h)+{3\eps\over(1-\g)^2}
\eqa
Together with Lemma~\ref{lem:aVvdQqd}ii this implies (i).
\qed\end{proof}

We are primarily interested in the optimal policy $\Pi^*(h)$; to
correctly represent the value $V^*(h)$ is only of indirect
interest. If $\Pi^*$ is $\phi$-uniform, it can be represented as
$\Pi^*(h)=\pi^0(s)$ for some $\pi^0$, but if the $\phi$-uniformity
condition on $V^*$ in Theorem~\ref{thm:aphiVstar} is dropped, the conclusion
$\Pi^*(h)=\pi^*(s)$ can fail as the following example shows.

\paradot{Counter Example}
Let $P$ be the MDP $P(o'r'|ha):=T^a_{oo'}\cdot[\![r'=R^a_o]\!]$ with two raw
states $o\in\{0,1\}$ and two actions $a\in\{\alpha,\beta\}$
formally defined on the left and depicted on the right:
\bqan
   && T^\alpha:=~\left({~1~~~0~\atop ~1~~~0~}\right),~~~~~~R^\alpha:=\left({1/6\atop 1}\right), \\[1ex]
   && T^\beta:=\left({1/2~~1/2\atop 1/2~~1/2}\right),~~~~~R^\beta:=\left({0\atop 1/2}\right),~~~~~~~~~~
\unitlength=3ex
\linethickness{0.4pt}
\begin{picture}(10,1)(0,1)
\thicklines
\put(1,3){\circle{3}}\put(1,3){\makebox(0,0)[cc]{\Large 0}}
\put(9,3){\circle{3}}\put(9,3){\makebox(0,0)[cc]{\Large 1}}
\put(1.8,4){\vector(1,0){6.4}}\put(1.7,4){\makebox(0,0)[lb]{$\beta,~r'=0,~p=\frs12$}}
\put(7.75,2.8){\vector(-1,0){5.5}}\put(5,2.8){\makebox(0,0)[cb]{$\alpha,~r'=1$}}
\put(8.2,2){\vector(-1,0){6.4}}\put(8.3,1.95){\makebox(0,0)[rt]{$\beta,~r'=\frs12,~p=\frs12$}}
\put(1,4.7){\oval(1.2,1.2)[tc]}\put(0.4,4.7){\vector(0,-1){0.55}}\put(1.6,4.7){\line(0,-1){0.55}}
\put(1,1.3){\oval(1.2,1.2)[bc]}\put(0.4,1.3){\vector(0,1){0.55}}\put(1.6,1.3){\line(0,1){0.55}}\put(0,0.8){\makebox(0,0)[lt]{$~\alpha,~r'=\frs16$}}
\put(9,1.3){\oval(1.2,1.2)[bc]}\put(8.4,1.3){\line(0,1){0.55}}\put(9.6,1.3){\vector(0,1){0.55}}
\end{picture}
\eqan
The value of policy $\pi$ in vector notation is $V^\pi=R^\pi+\g
T^\pi V^\pi$, where $V^\pi_{o_t}:=V^\pi(h_t)$. The 4 stationary
policies are denoted by $\pi=a^0 a^1$, where $a^o$ is the action
taken in raw state $o$. For $\g=0$, their values are
\beqn
  \begin{array}{c|cccc}
    \g=0,~\pi & \alpha\alpha & \alpha\beta & \beta\alpha & \beta\beta \\ \hline
    V^\pi_0=R_0^{\pi(0)} & 1/6 & 1/6 & 0 &  0 \\
    V^\pi_1=R_1^{\pi(1)} &  1  & 1/2 & 1 & 1/2
  \end{array}
\eeqn
Policy $\pi=\alpha\alpha$ has the highest value, therefore
$\Pi^*(h)\equiv\alpha$. Let us now aggregate raw states
$o\in\{0,1\}$ to a 1-state MDP. Its value is $v={1\over 1-\g}\rho^\trp
R = \rho^\trp R$ for $\g=0$, where $\rho$ is the stationary distribution
$\rho^\trp=\rho^\trp T$ of $T$, in particular $\rho^\alpha=({1\atop 0})$ for
$T^\alpha$ and $\rho^\beta=({1/2\atop 1/2})$ for $T^\beta$. Since
there is only 1 aggregated state, there are only 2 stationary policies, one
for each action. This leads to $v^\alpha=\fr16<\fr14=v^\beta$,
hence $\pi^*(s)\equiv\beta\neq\alpha\equiv\Pi^*(h)$ $\forall s,h$.
That is, despite $\Pi^*$ being constant, $\pi^*\neq\Pi^*$,
which shows that the condition on $V^*$ in Theorem~\ref{thm:aphiVstar} cannot be dropped.
Note that $V^*=V^{\alpha\alpha}=({1/6\atop 1})$ is far from constant.
By continuity, the policy reversal also holds for $\g>0$. Indeed,
this example works for all $\g<\fr25$ and other examples
work for all $0\leq\g<1$.\eoe

\begin{open}[\boldmath$\phi V*$]\label{open:aphiVstar}
Under the same conditions as Theorem~\ref{thm:aphiVstar}, is \\
\beq\label{eq:aphiVstar}
  V^*(h)-V^{\t\Pi}(h) ~\stackrel{??}=~ O\Big({\eps\over(1-\g)^?}\Big) \qmbox{where} \t\Pi(h):=\pi^*(s)
\eeq
\end{open}

\paradot{Arguments}
Here are some arguments why it might be true (or false):

(1) For $\eps=0$ it immediately follows from
Theorem~\ref{thm:aphiVstar}, since in this case
$\t\Pi(h)=\Pi^*(h)$. Some continuity argument might allow to
establish a bound for small $\eps>0$.

(2) Theorem~\ref{thm:aphiQstar}i\&iii mostly carried over to
Theorem~\ref{thm:aphiVstar}, so a-priori it is not too implausible
that Theorem~\ref{thm:aphiQstar}ii carries over to
\req{eq:aphiVstar}. On the other hand, the proofs of (i) and (iii)
of both theorems were sufficiently different, so the analogy
argument is weak.

(3) Let $\t a:=\t\Pi(h):=\pi^*(s)$ for $s=\phi(h)$. Then
\bqan
  \langle Q^*(h,\t a)\rangle_B &\stackrel{\!\!Thm.\ref{thm:aphiVstar}i\!\!}\geq& q^*(s,\t a)-{3\eps\g\over(1-\g)^2}
  ~\stackrel{\req{eqqstar}}=~ v^*(s)-{3\eps\g\over(1-\g)^2}
  ~\stackrel{\req{eqapVsph}}\geq~ V^*(h)-{3\eps\over(1-\g)^2}
\\
  Q^*(h,\t a) &\stackrel{\req{eqQstar}}\leq& V^*(h)
\eqan
For $\eps=0$ this pair of inequalities implies that $Q^*(h,\t a)$
lower bounds its own expectation, therefore it must be constant and equal to $V^*(h)$
on each $\phi^{-1}(s)$-partition.
For $\eps>0$, with high probability $Q^*(h,\t a)$ cannot be much smaller than $V^*(h)$.
If it weren't for the probability qualifier we could now apply Lemma~\ref{lem:aQpistar}
to establish \req{eq:aphiVstar} (as in the proof of Theorem~\ref{thm:aphiQstar}ii).
Low probability events could invalidate this argument.
\qed

\paradot{Discussion}
Open~Problem~\ref{open:aphiVstar} would be the main result if we had a proof for $\eps>0$.
Absent of it we have to be content with Theorem~\ref{thm:aphiQstar}ii.
Both statements imply that we can aggregate histories as much as we wish,
as long as the optimal value function and policy are still approximately
representable as functions of aggregated states.
Whether the reduced process $P_\phi$ is Markov or not is immaterial.
We can use surrogate MDP $p$ to find an $\eps$-optimal policy for $P$.

Most RL work, including on state aggregation, is formulated in terms of MDPs,
i.e.\ the original process $P$ is already an MDP. Let us call this
the original or raw MDP. We could interpret the whole history as a
raw state, which formally makes every $P$ an MDP, but normally only
observations are identified with raw states,
i.e.\ $P$ is a raw MDP iff $P(o'r'|ha)=P(o'r'|oa)$. In this case,
$V^*(h_t)=V^*(o_t)$ etc.\ depends on raw states only (which is
well known or follows from Theorem~\ref{thm:phiMDPstar} with
$\phi(h_t)=o_t$). Since our results hold for all $P$, they clearly
hold if $P$ is a raw MDP and if $\phi(h_t):=\phi(o_t)$ maps raw states
to aggregated states.

The remainder of this paper shows how much we can aggregate and how to
develop RL algorithms exploiting these insights.

\section{Extreme Aggregation}\label{sec:ExSAgg}

The results of Section~\ref{sec:AAResults} showed that histories
can be aggregated and modeled by an MDP even if the true aggregated process
is not an MDP. The only restrictions were that the (Q-)Value functions
and Policies could still be (approximately) represented as functions of the
aggregated states. We will see in this section that in theory this allows
to represent {\em any} process $P$ as a small finite-state MDP.

\paradot{Extreme aggregation based on Theorem~\ref{thm:aphiQstar}}
Consider $\phi$ that maps each history to the vector-over-actions of optimal
$Q$-values $Q^*(h,\cdot)$ discretized to some finite $\eps$-grid:
\beq\label{eqSQstar}
  \phi(h) ~:=~ \big(\lfloor Q^*(h,a)/\eps\rfloor\big)_{a\in\A}
          ~\in~ \{0,1,...,\lfloor\fr{1}{\eps(1-\g)}\rfloor\}^\A ~=:~ \S
\eeq
That is, all histories with $\eps$-close $Q^*$-values are mapped to the same state:
\beqn
  |Q^*(h,a)-Q^*(\t h,a)| ~\leq~ \eps ~~~ \forall \phi(h)=\phi(\t h)~\forall a
\eeqn
Now choose some $B$ and determine $p$ from $P$ via \req{eqapdef} and \req{eqPphi}.
Find the optimal policy $\pi^*$ of MDP $p$ of size $|\S|$.
Define $\t\Pi(h):=\pi^*(\phi(h))$. By Theorem~\ref{thm:aphiQstar}ii,
$\t\Pi$ is an $\eps'$-optimal policy of original process $P$ in the sense that
\beqn
  |V^{\t\Pi}(h) - V^*(h)| ~\leq~ {2\eps\over(1-\g)^2} ~=:~ \eps'
\eeqn

\paradot{Extreme aggregation based on Open~Problem~\ref{open:aphiVstar}}
If \req{eq:aphiVstar} holds, we can aggregate even better: Consider $\phi$ that
maps each history to the optimal Value $V^*(h)$ discretized to some
finite $\eps$-grid and to the optimal action $\Pi^*(h)$:
\beq\label{eqSVstar}
  \phi(h) ~:=~ \big(\lfloor V^*(h)/\eps\rfloor,\Pi^*(h)\big)
          ~\in~ \{0,1,...,\lfloor\fr{1}{\eps(1-\g)}\rfloor\}\times\A ~=:~ \S
\eeq
That is, all histories with $\eps$-close $V^*$-Values and same optimal action are mapped to the same state:
\beqn
  |V^*(h)-V^*(\t h)| ~\leq~ \eps \qmbox{and} \Pi^*(h)=\Pi^*(\t h) ~~~~~ \forall \phi(h)=\phi(\t h)
\eeqn
As before, determine $p$, find its optimal policy $\pi^*$,
and define $\t\Pi(h):=\pi^*(\phi(h))$. If \req{eq:aphiVstar} holds,
then $\t\Pi$ is an $\eps'$-optimal policy of original process $P$ in the sense that
\beqn
  |V^{\t\Pi(h)} - V^*(h)| ~=~ O\Big({\eps\over(1-\g)^?}\Big) ~=:~ \eps'
\eeqn
The following theorem summarizes the considerations for the two choices of $\phi$ above:

\begin{theorem}[Extreme \boldmath$\phi$]\label{thm:Exphi}
For every process $P$ there exists a reduction $\phi$
(\req{eqSQstar} or \req{eqSVstar} will do) and MDP $p$ defined  via
\req{eqapdef} and \req{eqPphi} whose optimal policy $\pi^*$ is an
$\eps'$-optimal policy $\t\Pi(h):=\pi^*(\phi(h))$ for $P$. The
size of the MDP is bounded (uniformly for {\em any} $P$) by
\beqn
  |\S|\leq\Big({3\over\eps'(1-\g)^3}\Big)^{|\A|}
  \qmbox{and if \req{eq:aphiVstar} holds even by} |\S|=O\Big({|\A|\over\eps'(1-\g)^{1+?}}\Big)
\eeqn
\end{theorem}

\begin{proof}
For $\S$ defined in \req{eqSQstar} we have
\beqn
  |\S| ~=~ \big(\lfloor\fr{1}{\eps(1-\g)}\rfloor\!+\!1\big)^{|\A|}
  ~=~ \big(\lfloor\fr{2}{\eps'(1-\g)^3}\rfloor\!+\!1\big)^{|\A|}
  ~\leq~ \big(\fr{3}{\eps'(1-\g)^3}\big)^{|\A|}
\eeqn
where in the last inequality we have assumed $\eps'\leq{1\over 1-\g}$.
(For $\eps'>{1\over 1-\g}$ the theorem is trivial, since any policy is $\eps'$-optimal).
For $\S$ defined in \req{eqSVstar} the derivation is similar.
The theorem now follows from the considerations in the paragraphs before the theorem.
\qed\end{proof}

\paradot{Discussion}
A valid question is of course whether Theorem~\ref{thm:Exphi} is
just an interesting theoretical insight/curiosity or of any
practical use. After all, $\phi$ depends on $Q^*$ (or $V^*$ and
$\Pi^*$), but if we knew $Q^*$, $\Pi^*$ would readily be available
and the detour through $p$ and $\pi^*$ pointless.

Theorem~\ref{thm:Exphi} reaches relevance by the following
observation: If we start with a sufficiently rich class of maps $\Phi$ that
contains at least one $\phi$ approximately representing $Q^*(h,\cdot)$, and have
a learning algorithm that favors such $\phi$, then
Theorems~\ref{thm:aphiQpi}--\ref{thm:aphiVstar} tell us that we do not need to
worry about whether $P_\phi$ is MDP or not; we ``simply'' use/learn
MDP $p$ instead. Theorem~\ref{thm:Exphi} tells us that this allows
for extreme aggregation far beyond MDPs.

This program is in parts worked out in the next two sections, but
more research is needed for its
completion. Learning $p$ from (real) $P$-samples is considered in
Section~\ref{sec:RL} and learning $\phi$ in Section~\ref{sec:FRL}.

\section{Reinforcement Learning}\label{sec:RL}

In RL, $P$ and therefore $p$ are unknown. We now show how to learn
$p$ from samples from $P$. For this we have to link $B$ to the
distribution over histories induced by $P$ and to the behavior
policy $\Pi_B$ the agent follows. We still assume $\phi$ is given.

\paradot{Behavior policy $\Pi_B$}
Let $\Pi_B:\H\leadsto\A$ be the behavior policy of our RL agent,
which in general is non-stationary due to learning, often
stochastic to ensure exploration, and (usually) different from any
policy considered so far ($\Pi^*\!\!,\,\pi^*\!\!,\,\t\Pi,\pi^0\!,\,\Pi,\pi$).
Note that a sequence of policies $\Pi_1,\Pi_2,...$ where each
$\Pi_t$ is learnt from $h_t$ and used at time $t$ (or for some
number of steps) is nothing but a single non-stationary policy
$\Pi_B(h_t)=\Pi_t(h_t)\forall t,h_t$, so $\Pi_B$ indeed includes
the case of policy learning.

\paradot{Choice of $B$}
The interaction of agent $\Pi_B$ with environment $P$ stochastically generates
some history $h_t$ followed by action $a_t$ with joint probability, say $P_B(h_t a_t)$.
We use subscripts $B$ and/or $\phi$ to indicate dependence on $\Pi_B$ and/or $\phi$.
A natural choice for $B(h|sa)$ in \req{eqaBdef} would be to condition of $P_B$ on $s_t a_t$.
We now show that this does not work and how to fix the problem.
We can get $P_{\phi B}(h_t|s_t a_t)$ from $P$ and $\Pi_B$ and several other useful distributions as follows:
\bqa
  \nonumber P_B(h_{t+1}|h_t) &=& P(o_{t+1}r_{t+1}|h_t a_t)\Pi_B(a_t|h_t) ~~~~~~~~~~ [h_{t+1}=h_t a_t o_{t+1}r_{t+1}] \\
  \nonumber P_B(h_n) &=& \prod_{t=0}^{n-1}P_B(h_{t+1}|h_t), ~~~~~~~~~~ P_B(h_t a_t) ~=~ \Pi_B(a_t|h_t)P_B(h_t) \\
  \nonumber P_{\phi B}(s_t a_t) &=& \sum_{\nq h_t:\phi(h_t)=s_t\nq\nq} P_B(h_t a_t), ~~~~~~~~~~~
            P_{\phi B}(h_t|s_t a_t) ~=~ {P_B(h_t a_t)\over P_{\phi B}(s_t a_t)}[\![\phi(h_t)=s_t]\!] \\
  \label{eqPsumPP} P_{\phi B}(s_{t+1}r_{t+1}|s_t a_t) &=& \sum_{\nq h_t:\phi(h_t)=s_t\nq\nq} P_\phi(s_{t+1}r_{t+1}|h_t a_t)P_{\phi B}(h_t|s_t a_t) ~~ [\text{see \req{eqPphi} for def.\ of $P_\phi$}] \\
  \nonumber P_{\phi B}(s_t a_t s_{t+1}r_{t+1}) &=& P_{\phi B}(s_{t+1}r_{t+1}|s_t a_t)P_{\phi B}(s_t a_t)
\eqa
$P_{\phi B}(h_t|s_t a_t)$ has the following properties:
\beq\label{eqPBhsa}
  P_{\phi B}(h_t|s_t a_t)\geq 0 \qmbox{and}
  \sum_{h_t\in\H_t}P_{\phi B}(h_t|s_t a_t) = \sum_{\nq h_t:\phi(h_t)=s_t\nq}P_{\phi B}(h_t|s_t a_t) = 1
  ~~\forall t,s_t,a_t
\eeq
This is close to the required condition \req{eqaBdef} for $B$ but crucially different.
The sum in \req{eqaBdef} is over histories of all lengths while in
\req{eqPBhsa} the sum is limited to histories of length $t$.
It is easy to miss this difference due to the compact notation.
Technically $P_B$ is a probability measure on infinite sequences
$\H_\infty$ and $P_B(h_t)$ is short for $P_B(\Gamma_{h_t})$ where
$\Gamma_{h_t}$ is the set of infinite histories starting with
$h_t$, i.e.\ $P_B(h_t)$ is the probability that the infinite history
starts with $h_t$ ($\sum_{h_t\in\H_t}P_B(h_t)=1\forall t$). On the
other hand, $B(h)$ is a probability distribution over finite
histories of mixed length ($\sum_{h\in\H}B(h)=1$); similarly for $P_B$ and $B$
conditioned on / parameterized by $s$ and $a$.

We can fix this mismatch by introducing weights $w_t:\S\times\A\leadsto[0;1]$ and define
\beq\label{eqBPB}
  B(h_t|sa) ~:=~ w_t(sa)P_{\phi B}(h_t|s_t=s,a_t=a)~\forall t,\qmbox{where} \sum_{t=1}^\infty w_t(sa)=1~\forall s,a
\eeq
which now satisfies \req{eqaBdef} (due to $\sum_{h\in\H}=\sum_{t=1}^\infty\sum_{h_t\in\H_t}$).
MDP $p$ can now be represented as
\bqa
  \nonumber p(s'r'|sa) &=& \sum_{t=1}^\infty w_t(sa)\sum_{\nq h_t\in\H_t\nq }P_\phi(s_{t+1}\!=\!s',r_{t+1}\!=\!r'|h_t,a_t\!=\!a)P_{\phi B}(h_t|s_t\!=\!s,a_t\!=\!a) \\
  \label{eqpPPB} &=& \sum_{t=1}^\infty w_t(sa) P_{\phi B}^t(s'r'|sa)
\eqa
That is, $p$ is the $w$-weighted time-average of $P_{\phi B}^t$.
The first equality follows from \req{eqapdef} and \req{eqBPB}; the
second one from \req{eqPsumPP}.
We also introduced the shorthand $P_{\phi B}^t(s'r'|sa):=P_{\phi B}(s_{t+1}=s',r_{t+1}=r'|s_t=s,a_t=a)$.

\paradot{Choice of $w_t$}
If $P_{\phi B}^t$ in \req{eqpPPB} is stationary, i.e.\ independent
of $t$, then $p(s'r'|sa)=P_{\phi B}^t(s'r'|sa)$ for all $t$, since the
weights sum to one, and estimation is easy. Note that in general
we cannot estimate non-stationary $P_{\phi B}^t$, since for each
$t$ we have only one sample available, but we will see that
estimation of $p$ is still possible.
Assume we have observed $h_n$, and choose
\beq\label{eqwtdef}
  w_t(sa) ~:=~ {P_{\phi B}^t(sa)\over \sum_{t=1}^n P_{\phi B}^t(sa)} \qmbox{for $t\leq n$ ~~and~~ $0$ ~~ for $t>n$}
\eeq
Inserting this into \req{eqpPPB} and using \req{eqPsumPP} gives
\beq\label{eqpPdP}
    p(s'r'|sa) ~=~ { \fr1n\sum_{t=1}^n P_{\phi B}^t(sas'r') \over \fr1n\sum_{t=1}^n P_{\phi B}^t(sa) }
\eeq
We estimate numerator and denominator separately.

\paradot{Law of large numbers}
For $t=1,2,3,...$ let $X_t\in\{0,1\}$ be binary random variables
with expectation $\E[X_t]$. Define
$n_1=\sum_{t=1}^n X_t$ be the number of sampled 1s. The strong law
of large numbers says that
\beq\label{eqLLN}
  {n_1\over n} - {1\over n}\sum_{t=1}^n\E[X_t] ~~\toinfty{n}~~ 0 \qmbox{almost surely ~~~ under weak conditions}
\eeq
Note that the law holds far beyond i.i.d.\ random variables under a
variety of conditions \cite{Fazekas:06,Vovk:05} which we collectively call `weak
conditions'. It is not even necessary for $n_1/n$ to converge.

\paradot{Estimation of $p$}
Now fix some $(s,a)$, and let $X_t:=[\![s_t=s,a_t=a]\!]$.
(Here we assume that variables in $h_t$ are random variables and
$sas'r'$ are realizations.) Then
\beqn
  n(sa) ~:=~ n_1 ~=~ \sum_{t=1}^n X_t ~=~ \#\{t\leq n:s_t=s,a_t=a\}
\eeqn
is the number of times action $a$ is taken in state $s$,
and $\E[X_t]=P(X_t=1)=P_{\phi B}^t(sa)$, hence \req{eqLLN} implies
\beq\label{eqLLNnsa}
   {n(sa)\over n} - {1\over n}\sum_{t=1}^n P_{\phi B}^t(sa) ~~\toinfty{n}~~ 0 \qmbox{a.s. under weak conditions}
\eeq
Similarly for $Y_t:=[\![s_t a_t s_{t+1}r_{t+1}=sas'r']\!]$ and $n(sas'r'):=\sum_{t=1}^n Y_t$
we have
\beq\label{eqLLNnsasr}
  {n(sas'r')\over n} - {1\over n}\sum_{t=1}^n P_{\phi B}^t(sas'r') ~~\toinfty{n}~~ 0 \qmbox{with $P$-probability 1}
\eeq
under weak conditions. \req{eqLLNnsa} and \req{eqLLNnsasr} via \req{eqpPdP} are nearly sufficient to
imply
\beq\label{eqpest}
  {n(sas'r')\over n(sa)} - p(s'r'|sa) ~~\xrightarrow{n\to\infty}~~ 0 \qmbox{almost surely}
\eeq
A sufficient but by far not necessary condition is
\beq\label{eqnsanz}
  \mathop{\lim\,\inf}\limits_{n\to\infty}{n(sa)\over n} ~>~0 \qmbox{almost surely}
\eeq

\begin{theorem}[$p$-estimation]\label{thm:pest}
For $B$ defined in \req{eqBPB} and \req{eqwtdef} we have:
If \req{eqLLNnsasr} and \req{eqnsanz} hold, then \req{eqpest} holds. %
For example, if $Y_t$ are stationary ergodic processes, then
\req{eqLLNnsasr} and \req{eqnsanz} hence \req{eqpest} hold for all
state-action pairs that matter (i.e.\ for those occurring with non-zero
probability).
\end{theorem}

\begin{proof}
We introduce the following ($n$-dependent) shorthands:
\bqan
  \bar X:={n(sa)\over n},~~~~~ & &
  \bar x:={1\over n}\sum_{t=1}^n P_{\phi B}^t(sa),~~~~
  \alpha:=\smash{\mathop{\lim\inf}\limits_{n\to\infty}}{n(sa)\over n},
\\
  \bar Y:={n(sas'r')\over n},~ & &
  \bar y:={1\over n}\sum_{t=1}^n P_{\phi B}^t(sas'r')
\eqan
With these abbreviations,
assumption \req{eqLLNnsasr} implies \req{eqLLNnsa}, i.e.
\beq\label{eqYyXx}
  \bar Y-\bar y ~\to~ 0 \qmbox{implies}
  \bar X-\bar x ~=~ \smash{\sum_{s'r'}\bar Y - \sum_{s'r'}\bar y
  ~=~ \sum_{s'r'}[\bar Y-\bar y]} ~\to~ 0
\eeq
since $\S$ and $\R$ have been assumed finite. Now
\bqan
  \Big|{n(sas'r')\over n(sa)} - p(s'r'|sa)\Big|
  &=& \Big|{\bar Y\over \bar X}-{\bar y\over \bar x}\Big|
  ~\leq~ \Big|{\bar Y\over \bar X}-{\bar y\over \bar X}\Big| ~+~ \Big|{\bar y\over \bar X}-{\bar y\over \bar x}\Big|
\\
  ~=~ {1\over \bar X}|\bar Y-\bar y| ~+~ {\bar y\over \bar X\bar x}|\bar x-\bar X|
  &\leq& {1\over \bar X}\Big(|\bar Y-\bar y| ~+~ |\bar x-\bar X|\Big)
  ~~\toinfty{n}~~ 0 \qmbox{a.s.}
\eqan
The first inequality is just the triangle inequality.
The second inequality follows from $\bar y\leq\bar x$.
The limit is zero, since almost surely
$\mathop{\lim\,\sup}_{n\to\infty}[1/\bar X]=1/\alpha<\infty$
and $\bar Y-\bar y\to 0$ and $\bar X-\bar x\to 0$.
Hence \req{eqpest} holds.
Finally, for stationary ergodic $Y_t$, we have $\bar y=\fr1n\sum_{t=1}^n
\E[Y_t]=\E[Y_1]=${\em constant}, and hence $\bar x=\sum_{s'r'}\bar
y=${\em constant}. Therefore
\bqan
  \req{eqLLNnsasr} & \text{holds by} & \bar Y ~=~ {1\over n}\sum_{t=1}^n Y_t ~\xrightarrow{\text{ergodicity}}~ \E[Y_1] ~\stackrel{\text{stationarity}}=~ {1\over n}\sum_{t=1}^n\E[Y_t] ~=~ \bar y, \\
  \req{eqnsanz} & \text{holds by} & \mathop{\lim\,\inf}\limits_{n\to\infty}\bar X
  ~\stackrel{\req{eqYyXx}}=~ \bar x
  ~\stackrel{\text{stationarity}}=~ \E[X_1]
  ~=~ P_{\phi B}^1(sa) ~\stackrel{\text{assumption}}>~ 0
\eqan\vspace*{-3ex}
\qed\end{proof}

\paradot{Discussion}
Limit \req{eqpest} shows that standard frequency estimation for $p$
will converge to the true $p$ under weak conditions. If $P_\phi$ is
MDP, samples are conditionally i.i.d.\ and the `weak conditions'
are satisfied. But the law of large numbers and hence \req{eqpest}
holds far beyond the i.i.d.\ case \cite{Fazekas:01,Vovk:05}, e.g.
for stationary ergodic processes. Condition \req{eqnsanz} that every
state-action pair be visited with non-vanishing relative frequency
can be significantly relaxed.
Stationarity is also not necessary, and indeed often does not hold
due to a non-stationary environment $P$ or a non-stationary
behavior policy $\Pi_B$ (or both).

Other choices for $w_t$ are possible, e.g.\ we could multiply
numerator and denominator of \req{eqwtdef} by some arbitrary
positive function $u_t(as)$, which leads to a weighted average estimator.

We estimate $p$ in order to estimate $q^*$ and ultimately $\pi^*$. This is model-based RL.
We can also learn $\pi^*$ model-free. For instance,
condition \req{eqpest} should be sufficient for Q-learning
to converge to $Q^*$.

Q-learning and other RL algorithms designed for MDPs have been
observed to often (but not always) perform well even if applied to
non-MDP domains. Our results appear to explain why, but this calls
for further investigations.

\section{Feature Reinforcement Learning}\label{sec:FRL}

The idea of FRL is to {\em learn} $\phi$ \cite{Hutter:09phimdpx}.
FRL starts with a class of maps $\Phi$, compares different
$\phi\in\Phi$, and {\em selects} the most appropriate one given the
experience $h_t$ so far. Several criteria based on how well $\phi$
reduces $P$ to an MDP have been devised
\cite{Hutter:09phimdp,Hutter:09phidbn} and theoretically
\cite{Hutter:10phimp} and experimentally \cite{Hutter:11frlexp}
investigated \cite{Nguyen:13phd}.
Theorems~\ref{thm:aphiQpi}--\ref{thm:aphiVstar} show that demanding
$P_\phi$ to be approximately MDP is overly restrictive.
Theorem~\ref{thm:Exphi} suggests that if we relax this condition,
much more substantial aggregation is possible, provided $\Phi$ is
rich enough.

(F)RL deals with the case of unknown $P$. We first discuss learning
$\phi$ for the unrealistic case of exact aggregation ($\eps=0$) and
infinite sample size ($n=\infty$). This serves as a useful guide to
work out its generalization to the realistic but significantly more
complex case of approximate aggregation based on finite sample
size. Finally we discuss a family of recent algorithms (BLB and
extensions \cite{Nguyen:13phd}) that appear to nearly have the
right properties for our purpose. This section is more
a collection of ideas and outlook towards algorithms
exploiting and motivating the usefulness of the new insights
obtained in the previous sections.

\paradot{Search for exact $\phi$ based on infinite sample size}
Since we are now concerned with comparing different $\phi\in\Phi$,
we subscribe quantities with $\phi$ when necessary. Consider the
unrealistic case of infinite sample size ($n=\infty$) and a search
for exact reductions $\phi$. We call a reduction
$\phi:\H\to\S_\phi$ exact iff $Q^*(h,a)=q^*_\phi(s,a)$ and
$\Pi^*(h)=\pi^*_\phi(s)$ for all $s=\phi(h)$ and $a$.

Even for $n=\infty$, $P$ hence $Q^*$ needed for $\Pi^*$ is
(usually) not estimable (from $h_\infty$). On the other hand, for
each $\phi\in\Phi$, $p=p_\phi$ can be determined (exactly) by
\req{eqpest} (under weak conditions).
From $p_\phi$ we can determine $q^*_\phi$ and $\pi^*_\phi$ via
\req{eqqstar} and \req{eqpistar}. The solution always satisfies the
reduced Bellman equations exactly, even for very bad reductions,
e.g.\ single state $\phi(h)\equiv 0\,\forall h$. So the reduced
problem is not sufficient to judge the quality of $\phi$.
An alternative to assuming $n=\infty$ is to assume that $P$ is
known, which also allows to determine $p_\phi$, etc. So what
follows applies to stochastic planning as well.

{\it Coarsening and refining reductions $\phi$:} Let us now coarsen
$\phi$, i.e.\ further merge some partitions $\phi^{-1}(s)$. In the
simplest case we just merge two states into one. In general, consider
coarsening $\chi:\S_\phi\to\S_\psi$ and coarser reduction
$\psi:\H\to\S_\psi$ such that $\psi(h)=\chi(\phi(h))$. We also call
$\S_\phi$ a refinement of $\S_\psi$.
For example, U-trees \cite{McCallum:96,Uther:98} and Kd-trees
have been used in RL \cite{Ernst:05}, where expanding a leaf
corresponds to splitting a state. Or in $\phi$DBN,
$S_\phi=\{0,1\}^d$ is a binary feature vector, where removing
one component corresponds to pairwise combining $2^d$ states
to $2^{d-1}$ states \cite{Hutter:09phidbn}.

{\it Ordering reductions in $\Phi$:}
We can partially order $\Phi$ as follows:
\bqan
  \psi\prec\phi &:\Leftrightarrow& \text{$q^*_\phi$ and $\pi^*_\phi$ are constant on all $s_\phi\in\chi^{-1}(s_\psi)$ for all $s_\psi$ and $a$} \\
  &\Leftrightarrow& \text{$q^*_\phi(s_\phi,a)=q^*_\psi(s_\psi,a)$ and $\pi^*_\phi(s_\phi)=\pi^*_\psi(s_\psi)$ for all $s_\psi=\chi(s_\phi)$ and $a$}.
\eqan
$\psi\prec\phi$ means $\psi$ is a better
reduction than $\phi$ since it leads to the same optimal $q$-value
and policy as $\psi$ does, but is more parsimonious (coarser) than $\phi$. If
$q^*_\phi$ or $\pi^*_\phi$ is not constant on $\psi$-partitions,
coarsening $\phi$ to $\psi$ and using $\psi$ (potentially)
leads to suboptimal solutions.

{\it Enriching the order $\prec$:} $\prec$ is a
transitive but `very' partial order. Two maps are
incomparable if neither is a refinement of the other. We can enrich
order $\prec$ as follows: For any two maps $\psi$ and $\psi'$, the
map $\phi(h):=(\psi(h),\psi'(h))\in\S_\phi=\S_\psi\times\S_{\psi'}$
refines both. Define $\psi\prec_\times\psi'$ iff
$\psi\prec\phi\prec\psi'$. Extended order
$\prec_\times$ is still not total. The remaining incomparable cases
are: Case $\psi\prec\phi\succ\psi'$: This is only
possible if $q^*$ and $\pi^*$ of $\psi$ and $\psi'$
(and $\phi$) coincide. A secondary criterion based on the relative
complexity of $\psi$ and $\psi'$ could decide the case, e.g.
$\psi\prec_\times\psi'$ iff $|\S_\psi|<|\S_{\psi'}|$. Case
$\psi\succ\phi\prec\psi'$: Both $\psi$ and $\psi'$ are inferior to
$\phi$. If class $\Phi$ is closed under cartesian product, $\phi$
should be favored over $\psi$ and $\psi'$ so their relative order
is not or less important.

{\it Search for $\phi$:}
Assume $\Phi$ contains at least one exact reduction.
Then the $\prec_\times$-minimal elements in $\Phi$ are exactly the
maximally coarse exact $\phi\in\Phi$.
If $\Phi$ is closed under arbitrary coarsening, then there is a unique minimizer (modulo isomorphism).
If $\Phi$ is also closed under cartesian product, the same holds for $\prec$.
This implies that any exhaustive search for a $\prec_\times$-minimum in $\Phi$ will
give an exact $\phi$ with minimal number of states, say $\phi_0$.
Now Theorem~\ref{thm:aphiQstar} tells us that $q^*_{\phi_0}$ and
$\pi^*_{\phi_0}$ are the optimal value and policy also of the
original process $P$, irrespective of whether $P_{\phi_0}$ is
Markov or not. So while the conditions of
Theorem~\ref{thm:aphiQstar} cannot be verified in practice, the
theorem justifies a search procedure based on
$(q^*_\phi,\pi^*_\phi)$ that ignores the (non-)Markov structure of
$P_\phi$.

\paradot{Search for approximate $\phi$ based on finite sample size}
The principle approach in the previous paragraph is sound,
but needs to be generalized in various ways before it can be used:
%
Real sample size is finite, which means we only have access to
approximations $\hat q^*_\phi$ and $\hat\pi^*_\phi$ via estimation
$\hat p_\phi$ of $p_\phi$. The criterion for exact equality
$q^*_\phi=q^*_\psi$ in $\prec$ needs to be replaced by a suitable
$\hat q^*_\phi\approx\hat q^*_\psi$, which should be done anyway,
since real-word problems seldom allow for exact reductions.
$\approx$ should be chosen so as to come with statistical
guarantees; e.g.\ Kolmogorov-Smirnov tests have been used in
\cite{McCallum:96}. A suitable
$\hat\pi^*_\phi\approx\hat\pi^*_\psi$ requires more effort (see
outlook). For large $\Phi$ this also requires appropriate
regularization, i.e.\ penalizing complex $\phi$
\cite{Hutter:09phimdpx}. To ensure $\hat q^*\to q^*$ for
$n\to\infty$, we need proper exploration strategies
\cite{Strehl:09}. Finally, we want an efficient search procedure in
$\Phi$, rather than exhaustive search. This will be heuristic or
will require strong assumptions on $\Phi$ \cite{Nguyen:13phd}. All but
the last point raised above have or should have general solutions
(see next paragraph).

\paradot{Utilizing existing algorithms}
The BLB algorithm \cite{Maillard:11} and its extensions IBLB
\cite{Nguyen:13aistats} and improvements OMS \cite{Maillard:13}
solve most of the problems above and can (nearly) readily be used
for our purpose.

The BLB family uses the same basic FRL setup from
\cite{Hutter:09phimdpx} used also here. The authors consider a
countable class $\Phi$ assumed to contain at least one
$\phi$ such that $P_\phi$ is an MDP \req{PphiMDP}. They consider
average reward, rather then discounting, and analyze regret, which
(in general) requires some assumption on the mixing rate or
`diameter' of the MDP. They prove that the total regret grows with
$\t O(n^{1/2...2/3})$, depending on the algorithm.

Their algorithms and analyses rely on UCRL2 \cite{Jaksch:10}, an
exploration algorithm for finite-state MDPs. Going through the BLB
proofs, it appears that the condition that $P_\phi$ is an MDP can
be removed if $p$ \req{eqapdef} is used instead, modulo the
analysis of UCRL2 itself. The proofs for the bounds for UCRL2
exploit that $s',r'$ conditioned on $s,a$ are i.i.d., which is true
if $P_\phi$ is Markov but not in general. Asymptotic versions
should remain valid under the `weak conditions' alluded to in
\req{eqpest}. With some stronger assumptions that guarantee good
convergence rates, the regret analysis of UCRL2 should remain valid
too. Formally, the use of Hoeffding's inequality for i.i.d.\ need
to be replaced by comparable bounds with weaker conditions, e.g.\
Azuma's inequality for martingales.

There is one serious gap in the argument above. BLB uses average
reward while our theorems are for discounted reward. It is often
possible to adapt algorithms and proofs which come with regret
bounds for average reward to PAC bounds for discounted reward or
vice versa. This would have to be done first: either a PAC version
of BLB by combining MERL \cite{Hutter:13pacgrl} with UCRL$\g$
\cite{Hutter:12pacmdp}, or average reward versions of the bounds
derived in this paper.

\section{Miscellaneous}\label{sec:Misc}

\paradot{Action permutation instead of policy condition}
We can rename actions without changing the underlying problem:
Let $A:\A\to\tilde\A$ be a bijection, and define
$\tilde P(o'r'|h\tilde a):=P(o'r'|ha)$, where $\tilde a:=A(a)$.
Clearly, all results for $P$ also hold for $\tilde P$ if
$a$ is replaced by $\tilde a$ everywhere, in particular
$\tilde\Pi(h):=A(\Pi(h))$. In general, this is of little use.
Things become more interesting if we allow the bijection $A$ to be history-dependent,
which we can do since our results hold for any, even non-stationarity, $\tilde P$.
This allows us to devise an $A:\A\times\H\to\tilde\A$
such that $A(\Pi(h);h)=${\it constant} for the policy $\Pi$ of interest.
For example, for $\tilde\A:=\A$, this is achieved by a permutation
that swaps action $a=\Pi(h)$ with some arbitrary but fixed action $a^1\in\A$,
and leaves all other actions unchanged:
\beqn
  A(a;h) := \left\{
              \begin{array}{lcl}
                ~a^1 & \hbox{if} & \Pi(h)=a \\
             \Pi(h) & \hbox{if} & \Pi(h)\neq a=a^1 \\
                ~a   & \hbox{else}
              \end{array}
            \right.
\eeqn
Since $\tilde\Pi(h)\equiv A(\Pi(h);h)\equiv a^1$ is constant, the
$\phi$-uniformity condition for $\tilde\Pi$ in
Theorems~\ref{thm:aphiQpi},~\ref{thm:aphiVpi}~and~\ref{thm:aphiVstar}
becomes vacuous.
While this transformation is theoretical interest, it only becomes
practically useful if we can somehow learn the function $A$ without
knowledge of $\Pi$, and in particular for $\Pi^*$. We could also
allow non-bijective $A$ that merge actions that have
(approximately) the same (optimal) $Q$-value.

\section{Discussion}\label{sec:Disc}

\paradot{Summary}
Our results show that RL algorithms for finite-state MDPs can be
utilized even for problems $P$ that have arbitrary history
dependence and history-to-state reductions/aggregations $\phi$ that
induce $P_\phi$ that are also neither stationary nor MDP. The only
condition to be placed on the reduction is that the quantities of
interest, (Q-)Values and (optimal) Policies, can approximately be
represented. This considerably generalizes previous work on feature
reinforcement learning and MDP state aggregation and allows for
extreme state aggregations beyond MDPs. The obtained results may
also explain why RL algorithms designed for MDPs sometimes perform
well beyond MDPs.

\paradot{Outlook}
As usual, lots remains to be done. A list of the
more interesting remaining tasks and open questions follows:

\blob While the approximate $\phi$-uniformity condition on $Q^*$ in
Theorem~\ref{thm:aphiQstar} is very weak compared to
bisimilarity, uniformity of $V^*$ in Theorem~\ref{thm:aphiVstar} is even weaker
(Theorem~\ref{thm:Exphi} shows how much of a difference this can make).
It is an Open~Problem~\ref{open:aphiVstar} whether an analogue
of Theorem~\ref{thm:aphiQstar}ii also holds for Theorem~\ref{thm:aphiVstar} beyond
$\eps=0$.

\blob An algorithm learning $\phi$ beyond MDPs that comes with regret or
PAC guarantees has yet to be developed. This could be done
by generalizing the partial order $\prec_\times$ to $n<\infty$, or
by adapting the class and proofs of BLB algorithms, or by integrating MERL with UCRL$\g$.
\blob All bounds contain ${1\over 1-\g}$ to some power. Can the exponents
be improved? For which environments/examples are the bounds tight?

\blob The trick to use $a$-dependent $Q^*$ as $a$-independent map
$\phi$ in Section~\ref{sec:ExSAgg} was to vectorize $Q^*$ in $a$.
Unfortunately this leads to a state-space size exponential in $\A$.
Solution $\phi$ based on $(V^*,\Pi^*)$ pair is only linear in $\A$,
but rests on Open~Problem~\ref{open:aphiVstar}. Are there other/better ways of
dealing with actions? Other extreme aggregations $\phi$, or are
$a$-dependent $\phi$ possible?

\blob Are average-reward total-regret versions of our discounted reward
results possible, under suitable mixing rate conditions?

\blob For small discrete action spaces typical for many board
games, the exact conditions on $\Pi$ are met. For continuous
action spaces as in robotics, we can simply discretize the action
space, introducing another $\eps$-error, but action-continuous
versions of our results would be nicer. Except for
Theorem~\ref{thm:aphiQstar}, any interesting generalization should
replace the exact by approximate $\phi$-uniformity conditions on
$\Pi$.

\blob Our theorems and/or proof ideas should allow to extend
existing convergence theorems for RL algorithms such as Q-learning
and others from MDPs to beyond MDPs.

\blob The bisimulation conditions of classical state aggregation results
are for reward and transition probabilities.
It would be interesting to derive explicit weaker conditions for them
that still imply our conditions on (Q-)Values.


\section*{References}\label{sec:Bib}
\addcontentsline{toc}{section}{\refname}

\def\refname{\vspace{-4ex}}
\bibliographystyle{alpha}
\begin{small}

\end{small}

\appendix \newpage
\section{List of Notation}\label{app:Notation}

\begin{tabbing}
  \hspace{0.13\textwidth} \= \hspace{0.73\textwidth} \= \kill
  {\bf General notation} \>                                                                  \\[0.5ex]
  $[\![R]\!]$            \> = 1 if $R$=true and =0 if $R$=false (Iverson bracket)                \\[0.5ex]
  $\#\cal X$         \> size of set $\cal X$                                                 \\[0.5ex]
  $\eps,\d$          \> small non-negative real numbers                                               \\[0.5ex]
  $\lfloor z\rfloor$ \> largest integer $\leq z$                                          \\[2ex]

  {\bf Original history-based process} \>                                                    \\[0.5ex]
  $\O,\R,\A$         \> = finite observation, reward, action spaces.                         \\[0.5ex]
  $o_t r_t a_t$      \> $\in\;\O\times\R\times\A$ = observation, reward, action at time $t$    \\[0.5ex]
  $t\leq n\in\SetN$  \> = any time $\leq$ sample size                                              \\[0.5ex]
  $P,Q,V,\Pi$        \> = Probability, (Q-)Value, Policy of original history-based Process    \\[0.5ex]
  $\Pi^*,\t\Pi$,$\Pi_B$ \> = optimal, approximately optimal, behavior Policy                 \\[0.5ex]
  $h\in\H$           \> = $(\O\times\R\times\A)^*\times\O\times\R$ = possible histories of any length\\[0.5ex]
  $h'=hao'r'$        \> = successor history of $h\in\H$                                      \\[0.5ex]
  $h_t$              \> = $o_1 r_1 a_1...o_t r_t$ = history up to time $t$                      \\[0.5ex]
  $\H_t$             \> = $(\O\times\R\times\A)^{t-1}\times\O\times\R$ = history of length $t$\\[0.5ex]
  $P(o'r'|ha)$       \> = probability of next observation\&reward given history\&action      \\[2ex]

  {\bf Reduction/aggregation from history to states} \>                                      \\[0.5ex]
  $\S_\phi$          \> = finite state space induced by $\phi$ (range of $\phi$)             \\[0.5ex]
  $\phi:\H\to\S_\phi$\> = reduction/map/aggregation from histories to states             \\[0.5ex]
  $s_t$              \> = $\phi(h_t)\in\S$ = state at time $t$                               \\[0.5ex]
  $P_\phi(s'r'|ha)$  \> = marginalized $P$-probability over state\&reward given history\&action\\[0.5ex]
  $B(h|sa)$          \> = dispersion probability. Stochastic ``inverse'' of $\phi$                 \\[0.5ex]
  $\langle Q(h,a)\rangle_B$ \> = $B$-average over $\{\t h:\phi(\t h)=\phi(h)\}$              \\[0.5ex]
  $w_t(sa)$          \> = non-negative weight function $\sum_{t=1}^\infty w_t(sa)=1~\forall sa$                \\[0.5ex]
  $P_B(h)$           \> = probability of $h$ from $P$ interacting with $\Pi_B$               \\[0.5ex]
  $P_{\phi B}$()     \> = (partially) $\phi$-reduced, marginalized, conditionalized $P_B$    \\[0.5ex]
  $\prec,\prec_\times$ \> = (extended) ordering of $\phi$ w.r.t.\ quality ($n=\infty$ so far only) \\[2ex]

  {\bf Finite state Markov decision process (MDP)} \>                                        \\[0.5ex]
  $\S$               \> = finite state space                                                 \\[0.5ex]
  $p,q,v,\pi$        \> = probability, (q-)value, policy of MDP                               \\[0.5ex]
  $s,a,s',r'$        \> = stat, action, successor state, reward                              \\[0.5ex]
  $n(sas'r')$        \> = number of times $sas'r'$ appears in $h_{n+1}$                          \\[0.5ex]
  $\g\in[0;1)$       \> = discount factor
\end{tabbing}

\end{document}